\documentclass[10pt,a4paper,reqno]{amsart}
\usepackage{amsmath,amsfonts,mathrsfs,amssymb,amsthm,mathtools}
\usepackage[shortlabels]{enumitem}

\usepackage[foot]{amsaddr}
\usepackage[toc, page]{appendix}

\usepackage[margin=1in]{geometry}

\usepackage{comment}

\usepackage[dvipsnames]{xcolor}
\definecolor{MyColor}{HTML}{0047AB}
\usepackage{fancyhdr}
\usepackage{hyperref}
\hypersetup{
colorlinks=true, 
linktoc=all,     
linkcolor=blue,  
citecolor=blue,
}

\usepackage{lipsum,bm}
\usepackage{amsfonts,wasysym}
\usepackage{graphicx}
\usepackage{epstopdf}
\usepackage{algorithmic}
\usepackage{amsopn}
\usepackage{amssymb}
\usepackage{mathtools}
\usepackage{verbatim}
\usepackage{enumitem}
\usepackage{hyperref}

\usepackage{mathrsfs}

\theoremstyle{plain}

\newtheorem{Theorem}{Theorem}[section]
\newtheorem{Lemma}{Lemma}[section]
\newtheorem{Proposition}{Proposition}[section]
 \newtheorem{Corollary}{Corollary}[section]
\newtheorem{Assumption}{H.\!\!}
\newtheorem{Example}{Example}[section]

\numberwithin{equation}{section}

\theoremstyle{definition}

\newtheorem{Definition}{Definition}[section]

\theoremstyle{remark}
\newtheorem{Remark}{Remark}[section]

\allowdisplaybreaks

\DeclareMathOperator*{\tr}{Tr}

\def\cB{\mathcal{B}}

\def\cF{\mathcal{F}}

\def\cN{\mathcal{N}}
\def\cO{\mathcal{O}}
\def\cP{\mathcal{P}}

\def\d{{ \mathrm{d}}}

\def\sE{{\mathbb{E}}}
\def\sF{{\mathbb{F}}}

\def\sN{{\mathbb{N}}}
\def\sP{\mathbb{P}}

\def\sR{{\mathbb R}}

\def\bb{\begin{equation}} \def\ee{\end{equation}}
\def\bbn{\begin{equation*}} \def\een{\end{equation*}}

\def\beqn{\begin{eqnarray}}  \def\eqn{\end{eqnarray}}

\def\beqnx{\begin{eqnarray*}} \def\eqnx{\end{eqnarray*}}

\def\bn{\begin{enumerate}} \def\en{\end{enumerate}}

\def\bd{\begin{description}} \def\ed{\end{description}}

\mathtoolsset{showonlyrefs}

\ifpdf
\DeclareGraphicsExtensions{.eps,.pdf,.png,.jpg}
\else
\DeclareGraphicsExtensions{.eps}
\fi

\newcommand{\change}[1]{{\color{black}#1}}

\newcommand{\E}{\mathbb{E}}
\newcommand{\Ew}{\mathbb{E}^W}

\newcommand{\R}{\mathbb{R}}
\newcommand{\N}{\mathbb{N}}

\newcommand{\var}[1]{\mathrm{Var}\left(#1\right)}

\newcommand{\dd}{\mathrm{d}}

\newcommand{\PP}{\mathbb{P}}

\newcommand{\aX}{\widetilde X}
\newcommand{\X}{X^{\mathscr{G}}}
\newcommand{\act}{a}

\newcommand{\A}{{A}}

\newcommand{\inti}{\int_{t_{i-1}}^{t_i}}

\newcommand{\bpi}{\boldsymbol{\pi}}

\makeatletter
\@namedef{subjclassname@2020}{\textup{2020} Mathematics Subject Classification}
\makeatother

\begin{document}
 
\title[Discretely Sampled Stochastic Policies for Continuous-time RL]{Accuracy of Discretely Sampled Stochastic Policies in Continuous-time Reinforcement Learning}

\author{Yanwei Jia$^1$}
\email{yanweijia@cuhk.edu.hk}
\author{Du Ouyang$^2$}
\email{oyd21@mails.tsinghua.edu.cn}
\author{Yufei Zhang$^3$}
\email{yufei.zhang@imperial.ac.uk}
\address{$^1$Department of Systems Engineering and Engineering Management, The Chinese University of Hong Kong, Hong Kong}
\address{$^2$Department of Mathematical Sciences, Tsinghua University, China}
\address{$^3$Department of Mathematics, Imperial College London, United Kingdom}



\begin{abstract}
Stochastic policies (\change{also known as relaxed controls}) are widely used in continuous-time reinforcement learning algorithms. However, executing a stochastic policy and evaluating its performance in a continuous-time environment remain open challenges. This work introduces and rigorously analyzes a policy execution framework that samples actions from a stochastic policy at discrete time points and implements them as piecewise constant controls. We prove that as the sampling mesh size tends to zero, the controlled state process converges weakly to the dynamics with coefficients aggregated according to the stochastic policy.
We explicitly quantify the convergence rate based on the regularity of the coefficients and establish an optimal first-order convergence rate for sufficiently regular coefficients. 
Additionally, we prove a $1/2$-order weak convergence rate that holds uniformly over the sampling noise with high probability, and establish a  $1/2$-order pathwise convergence for each realization of the system noise in the absence of volatility control.  Building on these results, we analyze the bias and variance of various policy evaluation and policy gradient estimators based on discrete-time observations. Our results provide theoretical justification for the exploratory stochastic control framework 
in  
[H.~Wang, T.~Zariphopoulou,  and  X.Y.~Zhou,    J.~Mach.~Learn.~Res., 21 (2020), pp.~1-34]. 
\end{abstract}

\maketitle


\section{Introduction}
\label{sec:introduction}

\subsection*{Continuous-time RL}
Reinforcement learning (RL) has catalyzed significant changes across various domains by attacking complex sequential decision-making challenges in a data-driven way. The theoretical framework for analyzing RL problems and designing RL algorithms has been predominantly based on the discrete-time Markov decision processes \cite{sutton2018reinforcement}. However, many practical applications require real-time interaction with environments where time progresses continuously, such as autonomous driving, stock trading, and vehicle routing. Recent studies \cite{tallec2019making,yildiz2021continuous,jia_q-learning_2022, giegrich2024convergence} have shown that certain algorithms originally designed for discrete-time RL can become unstable when applied in high-frequency settings. This highlights the need for a new theoretical framework for designing RL algorithms that are robust to time discretization.

To address this issue, many suggest formulating the problem in a continuous-time framework from the outset and only introducing time discretization at the final implementation stage. Here, time discretization refers to the frequency at which data is collected and actions are updated, while the underlying system itself evolves continuously according to its intrinsic dynamics. Early studies on continuous-time RL have primarily focused on deterministic systems, where state variables follow ordinary differential equations; see, e.g., \cite{baird1994reinforcement, doya2000reinforcement, tallec2019making, lee2021policy, kim2021hamilton}.

\change{
\subsection*{Stochastic policy
in continuous-time RL}

Recently, \cite{wang2020a} introduced a novel exploratory control framework for continuous-time RL   that fundamentally relies on the use of \emph{stochastic policies}. 
Stochastic policies are mappings from the state space to distributions over the action space, aligning with the notion of ``relaxed control" in classical control theory. 
By sampling actions from these stochastic policies, the agent assigns positive probabilities to different actions at each state, unlike the deterministic/strict policies typically employed in standard control problems.
This approach effectively captures the concept of exploration in RL, which is essential for the agents to learn about the environment.\footnotemark 
Since then, the use of stochastic policies and the corresponding exploratory control framework have become a foundational basis for developing a wide range of RL algorithms for continuous-time problems; see, e.g., \cite{jia2022a, jia2022b, jia_q-learning_2022, guo2022entropy, 
dai2023learning, frikha2023actor, wei2023continuous,  
zhao2023policy, giegrich2024convergence,  szpruch2024, sethi2025entropy}.

\footnotetext{\change{
Exploration in RL broadly refers to deliberately acting suboptimally to gather information about the environment. Stochastic policies achieve this by randomizing actions, but exploration can also be driven  by methods such as adding exploration bonuses to encourage actions in less-certain outcomes; see \cite[Chapter 2]{sutton2018reinforcement}. 
}}
 
} 
\subsection*{Execution of stochastic policies}

However, despite its popularity in algorithm design, the use of stochastic policies in continuous-time RL presents both theoretical and practical challenges. 
\change{
In particular, the behavior of the state dynamics under  frequent sampling from  a stochastic policy has not been rigorously analyzed. 

In \cite{wang2020a}, the authors assume that agents \emph{continuously} sample from a stochastic policy and \emph{heuristically} argue that the resulting dynamics follow a new diffusion process, with coefficients aggregated according to the sampling distribution of the policy. This aggregated dynamics offers a mathematically convenient framework for designing continuous-time RL algorithms. However, it requires generating uncountably many independent random variables, which is computationally infeasible and raises measure-theoretical concerns, rendering the controlled dynamics ill-defined. As a consequence, such an aggregated dynamics are not directly implementable. (see footnote \ref{footnote:measurability} for  details).

 In this paper, we consider a natural and practical approach to implement  stochastic policies. Rather than sampling actions continuously, we sample them at discrete time points and apply them to the continuous-time state system. As a result, the state dynamics are governed by a piecewise constant control process, whose evolution differs from the aggregated dynamics studied in \cite{wang2020a}. 
 
    For algorithms developed under the aggregated dynamics, such discrete-time execution of stochastic policies introduces a bias that must be carefully quantified.

In summary, this paper aims to address the following two questions:
\begin{description}
\item[Q1] 
Does the aggregated dynamics in \cite{wang2020a}  capture     the state dynamics induced by frequently sampling a stochastic policy? If so, in what sense, and what is the resulting   error?

\item[Q2] What is the precise impact of  the sampling frequency of a stochastic policy
on the performance  of the  corresponding learning algorithms? 
\end{description}

}

To the best of our knowledge, these questions have only been studied in the context of linear-quadratic (LQ) RL problems with Gaussian policies \cite{giegrich2024convergence,szpruch2024}. In this special setting, due to the LQ structure, the analysis is simplified by examining the discrepancy between the first and second moments of the aggregated dynamics and the controlled dynamics with sampled actions. Unfortunately, these techniques developed specifically for LQ RL problems are clearly not applicable to general RL problems. In particular, it is essential to quantify the difference between the entire distributions of two state processes.

\subsection*{Our contributions}
This paper analyzes the {convergence} of the discrete-time execution of stochastic policies for general diffusion processes and quantifies its impact on the associated learning algorithms.

\begin{itemize}
    \item 

We show that when actions are sampled from a stochastic policy at a discrete time grid
$\mathscr{G}$
and executed as a piecewise constant control process, the resulting state process, referred to as the \textit{sampled dynamics}, converges weakly to the aggregated dynamics as the time stepsize approaches zero. 
In particular, under sufficiently regular  coefficients, we show that the convergence rate is of first order \( O(|\mathscr{G}|) \), with \( |\mathscr{G}| \) being the time stepsize
(Theorem \ref{thm:weak convergence}). 
 \change{This matches the optimal convergence rate for standard piecewise-constant control approximations with smooth coefficients 
 \cite{jakobsen2019improved}}. 
We further quantify the precise weak convergence rate for cases with less regular coefficients
(Theorem \ref{thm:weak convergence relax regularity}).

The convergence results rigorously justify the exploratory control framework proposed in \cite{wang2020a} (see Section \ref{sec:setup}). The weak convergence rate will be used to analyze the performance of learning algorithms based on observations from the sampled dynamics.

\item 
We prove that the sampled dynamics, in fact, converges \emph{almost surely} with respect to the randomness in policy execution.
Specifically, we show that
the sampled dynamics converge   at order
 \(\mathcal{O}(|\mathscr{G}|^{1/2})\), averaged over the system noise but  uniformly across the realizations of the sampling noise with high probability  (Theorem \ref{thm:concentration inequality}).
We further show that 
when the volatility is uncontrolled,
 the sampled dynamics converge at order  \(\mathcal{O}(|\mathscr{G}|^{1/2})\) \emph{for almost every  realization of the system noise} (Theorem \ref{thm:strong conv}).
We show in  Example \ref{eg:counter example}  that the condition of uncontrolled volatility is necessary for the convergence of the sampled dynamics for each fixed system noise.

 These almost sure convergence results allow   for interpreting the convergence of the sampled dynamics to the aggregated dynamics as   \emph{a law of large numbers in time}, i.e.,   convergence is achieved  by increasing the sampling frequency without   averaging over multiple realizations of the sampling noises at given time points.
  \change{See Corollary \ref{cor:concentration inequality}
  and Remark \ref{rmk:sample_complexity}
  for its implications on the improved sample complexity of associated Monte Carlo estimators. }

\item We apply the weak convergence rate to quantify the error incurred when implementing continuous-time RL algorithms based on observations from sampled dynamics.  In particular, we show that evaluating the reward involving aggregated dynamics by discrete sampling and piecewise constant execution has a bias of the order $O(|\mathscr{G}|)$.
Calculating the martingale orthogonality conditions proposed in \cite{jia2022a} using discretely sampled stochastic policies also introduces a bias of the same order, \( O(|\mathscr{G}|) \). The policy gradient representation developed in \cite{jia2022b} is an informal integration that involves the realized actions and the state variables. We rigorously show that using the forward Euler scheme to calculate such terms with discretely sampled stochastic policies is a biased estimator for the policy gradient, and the bias is $O(|\mathscr{G}|)$ and variance is bounded. The q-learning with quadratic variation penalty for the risk-sensitive RL proposed in \cite{jia2024continuous} is also shown to have a similar order of discretization error.
\end{itemize}

\subsection*{Our approaches and key technical challenges}

\change{
Our convergence rate analysis for the discrete-time execution of stochastic policies requires addressing several technical issues beyond those encountered in existing literature.
These challenges arise because the sampled dynamics contain two distinct sources of randomness: one from control randomization and the other from the underlying system noise.

Specifically, the sampled dynamics have random coefficients due to control randomization, and they converge to a limiting diffusion process with a different set of deterministic coefficients, resulting from   aggregation over the stochastic policy. Quantifying this convergence rate has not been addressed in the literature, and it is fundamentally different from the classical analysis of time discretizations for stochastic differential equations (SDEs) (see, e.g., \cite{kloeden1992stochastic}). Specifically, standard numerical analysis of SDEs focuses on the error caused by discretizing the system noise,  while both the SDE and its time-discretized version share \emph{the same coefficients}. In our setting, the environment noise is not discretized; instead, the error arises from the control randomization.  

Moreover, we quantify the convergence rate of the sampled dynamics under different criteria, including the weak error with respect to both system and sampling noises, the weak error conditioned on the sampling noise, and almost sure convergence with respect to both noises. Identifying precise conditions to ensure convergence under these criteria requires carefully distinguishing the roles of system and sampling noises and employing different analytical techniques, including the regularity theory of partial differential equations (PDEs) and various martingale concentration inequalities.

In particular, when quantifying the weak error conditioned on the sampling noise (Theorem \ref{thm:concentration inequality}), we leverage the averaging over the system noise and apply the Azuma–Hoeffding inequality for bounded martingales. For the almost sure convergence with respect to both noises (Theorem \ref{thm:strong conv}), more precise estimates involving  conditional Orlicz norms
of sub-exponential  martingales
are required to carefully characterize the concentration behavior with respect to  both system and sampling noises. To the best of our knowledge, this is the first work that applies concentration inequalities to the convergence rate analysis of sampling schemes for SDEs. 
}

\subsection*{Most related works}
This paper quantifies the accuracy of the discrete-time execution of a stochastic policy and the performance of continuous-time RL algorithms based on discrete-time observations. The accuracy of policy execution has been studied in \cite{guo2023reinforcement,szpruch2024} for LQ problems with Gaussian policies; however, their techniques do not extend to the RL problems with general state dynamics and stochastic policies as considered here. 
Regarding the performance of continuous-time RL algorithms using discrete-time observations, \change{most existing works assume that the observations are generated directly from the desired continuous-time state dynamics,  independent of time discretizations (see, e.g., \cite{basei2022logarithmic, jia2022a, zhu2024PhiBEa, zhu2025OptimalPhiBE}). This assumption no longer holds in the present setting with stochastic policies, since the agent observes the sampled dynamics, which differ from their continuous-time limit.}

For general state dynamics  and general stochastic policies, the weak convergence of sampled dynamics has recently been studied in \cite{bender2024random, bender2024grid}. In their setting, the policy is evaluated continuously at the current state, and weak convergence is established by interpreting the sampled action as a random measure, though without providing a convergence rate. In contrast, we consider a more practical setting where the policy is evaluated only at discrete time points, so that the sampled dynamics are governed by a piecewise constant control process. Moreover, we establish not only the weak convergence rate, leveraging the regularity theory of PDEs, but also the high-probability bound and almost sure convergence rate with respect to the policy execution noise.

For state dynamics with  drift controls only, 
  \cite{carmona2025reconciling}
  establishes the strong convergence of time-discretized sampled dynamics under the assumption that the drift coefficient is bounded. This boundedness condition is restrictive, as it rules out the linear dynamics commonly encountered in many RL problems \cite{basei2022logarithmic,guo2023reinforcement,szpruch2024}. Our result removes this limitation by allowing drift coefficients with linear growth. To address the challenges posed by unbounded coefficients, we employ advanced sub-exponential concentration inequalities with conditional Orlicz norms to quantify the convergence rate.

\subsection*{Organization of the paper}
  Section \ref{sec:setup} briefly reviews the 
use of stochastic policies in continuous-time RL. Section \ref{sec:sampled dynamics} presents the problem setup and introduces the sampled dynamics rigorously. Section \ref{sec:convergence sampled state} presents the main convergence results
of the sampled dynamics,
and Section \ref{sec:applications}
applies them to analyze continuous-time RL algorithms. 
Section \ref{sec:proofs}
presents the proofs of the main results. Appendix \ref{sec proof: lemma} contains the proofs of useful Lemmas and intermediary results.

\section{Preliminary}
\label{sec:setup}
In this section, we review the use of stochastic policies in continuous-time RL problems, \change{while avoiding unnecessary technical details. 
The precise assumptions and setup are  given   in Section \ref{sec:sampled dynamics}.}

\subsection{Classical stochastic control}
\label{sec:classical control}

Recall that in a classical stochastic control problem, an agent aims to control the state process and optimize an objective functional through a  \change{non-anticipative}   control process taking values in a Polish space $A$.
Specifically, given a \change{non-anticipative} control  process $\{a_t\}_{t\ge 0}$, the corresponding state process $\{X_t\}_{t\ge 0}$  is governed by the following dynamics:\footnote{To simplify the presentation, we assume that the state processes have a deterministic initial condition. The analysis and results extend naturally to cases with a random initial condition that  is independent of the Brownian motion $W$ and satisfies appropriate integrability conditions. We also assume without loss of generality that the Brownian motion and the state variable have the same dimensions.}
\begin{equation}
\label{eq:sde}
\dd X_t = b(t,X_t,a_t) \dd t +  \sigma(t,X_t,a_t) \dd W_t,  \quad t\in [0,T]; 
\quad X_0 =x_0, 
\end{equation}
where
$x_0\in \sR^d$ is a given initial state, 
$b:[0,T]\times \mathbb R^d \times A \to \mathbb R^d$ and $\sigma: [0,T]\times \mathbb R^d \times A \to \mathbb R^{d\times d}$ are given  drift and volatility coefficients, respectively,
and 
$W$ is a $d$-dimensional Brownian motion supported on  a  probability space \((\Omega, \mathcal{F}, \PP)\).
The agent's objective is to maximize the following objective 
\begin{equation}
\label{eq:classical objective}
\E\left[ \int_0^T r(t,X_t,a_t)\dd t + h(X_T) \right]
\end{equation}
over all 
non-anticipative controls,  where   $r:[0,T]\times \sR^d\times A\to \sR$ is a given instantaneous reward function, and $h:\sR^d\to \sR$ is a  given  terminal reward function.  
\change{We refer the reader to \cite[Chapter 2, Section 4]{YZbook} for a more detailed formulation of stochastic control problems.}

It is known that \change{under mild regularity  conditions (see e.g., \cite{nicole1987compactification, kurtz1998existence})},
an optimal control of  \eqref{eq:classical objective} exists and can be expressed in feedback form. That is, there exists a function  $\bm u:[0,T]\times \mathbb R^d \to A$, called a policy,  
such that the optimal control is obtained by $a_t = \bm u(t, X_t^{\bm u})$ for all $t\in [0,T]$, where $\{X_t^{\bm u}\}_{t\ge 0}$ satisfies the following dynamics:
\begin{equation}
\label{eq:sde_deterministic_policy}
 \dd X_t = b(t,X_t,\bm u(t, X_t)) \dd t +  \sigma(t,X_t,\bm u(t, X_t)) \dd W_t,
  \quad t\in [0,T]; 
\quad X_0 =x_0,   
\end{equation}
The policy 
  ${\bm u}$  
can be determined using the coefficients $b, \sigma, r,$ and $h$. Once determined, the optimal control can be implemented by evaluating ${\bm u}$ at the current state of the system.

\subsection{ Stochastic policy in RL and the associated state dynamics}
\label{sec:exploratory dynamics} 

In the RL problem, the agent seeks to optimize the objective \eqref{eq:classical objective} by directly interacting with the dynamics \eqref{eq:sde},
experimenting with different actions, and refining their strategies based on the observed responses. Unlike traditional control settings, the agent treats \eqref{eq:sde} as a black-box environment evolving in continuous time and learns an optimal policy solely from observations, without direct access to the coefficients $b, \sigma, r,$ and $h$.

To avoid getting trapped in local optima, a key feature of RL is explicitly exploring the environment. A common approach is to use \emph{stochastic policies}  $\bpi:[0,T]\times \sR^d\to \cP(A)$, which maps the current state to a probability distribution over the action space. Given a stochastic policy $\bpi$, the agent interacts with the state  by sampling noisy actions from $\bpi$ and observing the resulting   state transitions. Over time, the agent refines the stochastic policy to optimize the objective.

\change{A crucial step in designing learning algorithms with stochastic policies is to characterize the performance of a stochastic policy $\bpi$.}   
\cite{wang2020a} suggests that given the policy $\bpi$, the associated state process $\aX$ evolves according to  the following dynamics:
\begin{equation}
\label{eq:aggregated_sde}
\d   X_t = \tilde b^{\bpi}(t, X_t )\dd t + \tilde \sigma^{\bpi}(t, X_t )\dd W_t,
\quad t\in [0,T]; 
\quad X_0 =x_0,
\end{equation}
with the   coefficients  $\tilde b^{\bpi}:[0,T]\times \sR^d\to \sR^d $ and $\tilde \sigma^{\bpi}: [0,T]\times \sR^d\to \sR^{d\times d} $    given by
\begin{equation}
	\label{eq:coef with star}
	\begin{aligned}
		& \tilde b^{\bpi}(t,x) \coloneqq \int_{A} b(t,x,a)\bpi(\dd a|t,x)  ,\quad \tilde \sigma^{\bpi}(t,x) \coloneqq \sqrt{\int_{A} (\sigma\sigma^\top)(t,x,a)\bpi(\d a|t,x) },
	\end{aligned}
\end{equation}
where $(\cdot)^{1/2}$
is the matrix square root of positive semidefinite matrices, and $(\cdot)^\top$ is the matrix transpose.
Moreover, the associated objective function is given by:
\begin{equation}
\label{eq:objective entropy relaxed control}
\E\left[ \int_0^T  \tilde r(t,\aX_t)\dd t + h(\aX_T) \right],
\quad \textnormal{with
$\tilde r(t,x) \coloneqq  \int_{A} r(t,x,a)\bpi(\dd a| t,x)$.
}   
\end{equation}
In the sequel, we call \eqref{eq:aggregated_sde} the \emph{aggregated dynamics} to emphasize that its coefficients are obtained by averaging the coefficients of \eqref{eq:sde} with respect to the sampling distribution defined by $\bpi$.

The dynamics \eqref{eq:aggregated_sde} and the objective \eqref{eq:objective entropy relaxed control} are derived \emph{heuristically} in \cite{wang2020a} under the assumption that the agent interacts with the system \eqref{eq:sde} by \emph{continuously} sampling 
\emph{independent}
random actions from the policy \( \bpi \). Under this assumption, the authors determine the state coefficients \eqref{eq:coef with star} and the instantaneous reward \eqref{eq:objective entropy relaxed control} by examining the first and second moments of the infinitesimal state change and applying a ``law-of-large-numbers'' heuristic.
\change{However,  the aggregated dynamics \eqref{eq:aggregated_sde} differs from the practically observed state process, since realizing it would require continuously sampling independent actions.\footnotemark}

\change{A more natural and practical implementation of a stochastic policy is to sample actions only at discrete time points and apply   them to   the state process   \eqref{eq:sde} (see \cite{giegrich2024convergence,szpruch2024}). The resulting state dynamics is governed by \eqref{eq:sde} with a piecewise constant control process, whose evolution differs from the aggregated dynamics \eqref{eq:aggregated_sde}. In the sequel, we show that these state dynamics converge to the aggregated dynamics \eqref{eq:aggregated_sde} as the sampling frequency increases, and we quantify the corresponding convergence rates. We also analyze the impact of this approximation on algorithms designed based on \eqref{eq:aggregated_sde}.}

\footnotetext{In fact, 
 interacting with 
the SDE \eqref{eq:sde} by \emph{continuously} sampling from a stochastic policy creates measure-theoretical issues. 
As pointed out in \cite[Remark 2.1]{szpruch2024},
it is impossible to construct a family of non-constant  random variables $(\xi_t)_{t\in [0,1]}$
such that $(\xi_t)_{t\in [0,1]}$ is (essentially) pairwise independent 
and $t\mapsto \xi_t$ is Lebesgue measurable. 
This implies that if one controls \eqref{eq:sde} by continuously generating independent actions, the resulting coefficients are not progressively measurable, rendering the conventional stochastic integral ill-defined.
\label{footnote:measurability}
 }

\section{Probabilistic Setup for Aggregated and Sampled Dynamics}
\label{sec:sampled dynamics}

This section rigorously formulates   the aggregated dynamics under a given stochastic policy and the associated sampled dynamics with a finite number of intervention times, as outlined in Section \ref{sec:setup}.

\subsection{Aggregated dynamics}
\label{sec:probabilistic setup}

Fix a probability space \((\Omega^W, \mathcal{F}^W, \PP^W)\) on which a $d$-dimensional Brownian motion $\{W_t\}_{t\ge 0}$ is defined, and let $\sF^W\coloneqq \{\mathcal{F}_t^W\}_{t \geq 0}$ 
be the $\PP^W$-completion of the filtration generated by $W$, which satisfies the usual conditions of completeness and right-continuity. 
\change{Let the action set $A$
 be a nonempty Polish space equipped with the metric $d_A$, and let  $\mathcal{B}(\A)$ be the Borel $\sigma$-algebra on $A$. }
Throughout this paper, we 
 impose the following regularity condition on the drift and diffusion coefficients $b$ and $\sigma$,
and the stochastic policy $\bpi$.

\begin{Assumption}
\label{assum:standing}
\begin{enumerate}[(1)]
  
\item
\label{item:b_sigma}
$b:[0,T]\times \R^d \times \A\to \R^d $  
and $\sigma: [0,T]\times \R^d \times \A\to \R^{d\times d}$
are measurable functions. There exists $C\ge 0$
such that for a fixed $a_0\in A$, any 
$\varphi\in \{b,\sigma\}$, $t\in [0,T]$, $x,x'\in \R^d$ and $a\in \A$,
\begin{align*} 
|\varphi(t,x,a)-\varphi(t,x',a)|  \le C|x-x'|,
\quad |\varphi(t,x,a)| &\le C(1+|x|+d_\A(a,a_0));
\end{align*} 
\item 
\label{item:pi}
$\bpi:[0,T]\times \R^d\times \cB(\A)\to [0,1]$ is a probability kernel.
For all $p\ge 2$, there exists 
$C_p\ge 0$  such that 
for all $(t,x)\in [0,T]\times \sR^d$,
$\int_\A d_A(a,a_0)^p \bpi(\d a|t,x)\le C_p(1+|x|^p)$, with a fixed $a_0\in A$.

\end{enumerate}

\end{Assumption}

Condition (H.\ref{assum:standing}\ref{item:b_sigma}) imposes the standard Lipschitz regularity on the system coefficients commonly used in the stochastic control literature (see, e.g., \cite{fleming2006controlled, YZbook}), and Condition (H.\ref{assum:standing}\ref{item:pi}) requires mild growth and integrability properties for the stochastic policy.
Under (H.\ref{assum:standing}),
$\tilde b^{\bpi}$ and $\tilde \sigma^{\bpi}$ in \eqref{eq:coef with star} are well-defined measurable functions and satisfy 
$
|\tilde{b}^{\bpi}(t,x)|+|\tilde{\sigma}^{\bpi}(t,x)| \le C(1+|x|)
$ for all $(t,x)\in [0,T]\times \sR^d$.

We assume that  $\tilde{b}^{\bpi}$ and $\tilde{\sigma}^{\bpi}$ are sufficiently regular so that the aggregated dynamic \eqref{eq:aggregated_sde} is well-posed.

\begin{Assumption}
\label{assum:Lipschitz}
The functions  $\tilde b^{\bpi}$ and $\tilde \sigma^{\bpi}$ are sufficiently regular so that \eqref{eq:aggregated_sde} admits a unique strong solution $\tilde X$ on the filtered  probability space $(\Omega ^W, \mathcal F^W, \sF^W, \PP^W)$.
\end{Assumption}

Here, we focus on strong solutions of \eqref{eq:aggregated_sde} rather than weak solutions, 
which allows us to identify  $W$  as the only driving noise in the state process 
 $\tilde X$. This simplifies the construction of the sampled dynamics and the subsequent error estimates in terms of the sampling frequency.
Under (H.\ref{assum:standing}),
Condition (H.\ref{assum:Lipschitz})
holds if $\tilde b^{\bpi}$ and $\tilde \sigma^{\bpi}$  
are locally Lipschitz continuous in the state variable 
(cf. \cite[Theorem 3.4]{mao2008stochastic}) and can be ensured via suitable regularity conditions on $\bpi$ (e.g., \cite[Lemma 2]{jia2022b}). Some typical sufficient conditions on $b, \sigma$ and $\bpi$ to guarantee these regularity conditions on $\tilde b^{\bpi}$ and $\tilde \sigma^{\bpi}$ are given in Examples \ref{example:general case}, \ref{example:guassian} and \ref{example:degenerate}.

\subsection{Sampled dynamics with randomized actions}
\label{sec:algo discretization}

We proceed to define the state process with random actions sampled according to the stochastic policy $\bpi$. Rather than focusing on a specific sampling procedure for $\bpi$, we adopt a more abstract formulation of the sampling process defined as follows. 
\begin{Definition}
\label{def:sample}
Assume the notation of (H.\ref{assum:standing}). 
We say a tuple $(\Omega^\xi, \cF^\xi, \sP^{\xi}, E, \xi, \phi) $ a sampling procedure of the policy $\bpi$ if $ (\Omega^\xi, \cF^\xi, \sP^{\xi})$ is a complete probability space, 
$(E,\mathcal{B}(E))$  is a Borel space, $\xi:\Omega^\xi\to E$ is a random variable and $\phi :[0,T]\times \sR^d\times E\to A$
is a measurable function such that for all $(t,x)\in [0,T]\times \sR^d$, $\Omega^\xi: \omega\mapsto \phi(t,x,\xi(\omega))\in A$ has the distribution $\bpi(\d a|t,x)$ under the measure $\sP^\xi$. 
\end{Definition}

By Definition \ref{def:sample}, $(E, \xi, \phi)$ provides a framework for \emph{executing} the policy $\bpi$ by sampling a random action $a_t\coloneqq \phi(t,X_t, \xi)$ from the distribution $\bpi(\d a | t,X_{t})$ at a given time $t\in [0,T]$
and state $X_{t}$. This action, conditioned on the $\sigma$-algebra $\sigma\left\{(X_{s})_{s \in [0,t]}\right\}$,
is independent of both $(X_{s})_{s \in [0,t ]}$ and the underlying Brownian motion $W$.

\begin{Remark}
As $A$ is a Polish space, a sampling procedure of a probability kernel $\bpi$ always exists. Indeed, by \cite[Lemma 2.22]{kallenberg2002foundations}, it suffices to take $E=[0,1]$ and $\xi$ as a  uniform  random variable on $E$. This approach has been utilized in  \cite{frikha2023actor, bender2024random, bender2024grid}. However, the corresponding function $\phi$ is guaranteed to exist only through a measure-theoretic argument, making it difficult to implement such a sampling procedure. 

Here we work with a general Borel space $E$, instead of restricting $E$ to $[0,1]$. This generalization accommodates commonly used sampling procedures that generate actions using multiple independent normal random variables, as seen in applications such as multivariate Gaussian policies  \cite{szpruch2024} or suitable Markov chain Monte Carlo methods. In Section \ref{sec:convergence sampled state}, we establish the weak convergence of the control-randomized SDE 
under the mere measurability assumption of $\phi$. Additional regularity of the function $\phi$  will be imposed to obtain strong convergence results.

\end{Remark}
Now, fix a sampling procedure
$(\Omega^\xi, \cF^\xi, \sP^{\xi}, E, \xi, \phi) $ of $\bpi$,
let $\sN_0 =\sN\cup\{0\}$
and let $(\Omega^{\xi_n}, \cF^{\xi_n}, \sP^{\xi_n},   \xi_n)_{n\in \sN_0}$
be independent copies of $(\Omega^\xi, \cF^\xi, \sP^{\xi},   \xi) $. Consider a probability space of the following form:
\begin{align}
\label{eq:space}
(\Omega, \cF,\sP)\coloneqq \bigg(\Omega^W \times  \prod_{n=0}^\infty \Omega^{\xi_n}, \cF^W\otimes \bigotimes_{n=0}^\infty 
\cF^{\xi_n}, 
\sP^W \otimes 
\bigotimes_{n=0}^\infty
\sP^{\xi_n}
\bigg),
\end{align}
where $(\Omega^W,\mathcal F^W, \mathbb P^W)$ supports   
the Brownian motion $W$ in \eqref{eq:aggregated_sde},
and for each $n\in \sN_0$, $(\Omega^{\xi_n}, \cF^{\xi_n}, \sP^{\xi_n})$ supports the random variable $\xi_n$  used to generate the random  control at the grid point $t_n$. With a slight abuse of notation, we complete the probability space $(\Omega, \cF,\sP)$ by augmenting  $\cF$
with $\sP$-null sets, and extend canonically  any random variable  on a marginal space to the whole space $\Omega$, e.g., $W(\omega^W,\omega^\xi )=W(\omega^W)$
and $\aX(\omega^W,\omega^\xi )=\aX(\omega^W)$
for all $\omega^W\in \Omega^W$ and $\omega^\xi \in \prod_{n=0}^\infty \Omega^{\xi_n}$.

Given a time grid $\mathscr G= \{0=t_0<\ldots <t_n =T\} $   of $[0,T]$, we consider interacting with the state dynamics \eqref{eq:sde} by sampling actions at the grid points in $\mathscr G$ according to the policy $\bpi$. More precisely, 
we consider the dynamics such that for all $i=0,\ldots, n-1$ and all $t\in [t_i,t_{i+1}]$,
\begin{align}\label{eq:sampled_sde}
\begin{split}
X_t & = X_{t_i} + \int_{t_i}^t b(s, X_s,a_{t_i}) \d s + 
\int_{t_i}^t \sigma(s, X_s,a_{t_i})\d  W_s,
\quad 
\textnormal{with $a_{t_i}= \phi(t_i,X_{t_i},\xi_i)$};
\quad X_0=x_0,
\end{split}   
\end{align}
which will be referred to as the \emph{sampled dynamics}\footnote{This definition slightly differs from the definition of the ``grid sampling SDE" in \cite[Section 3]{bender2024grid} even though both rely on drawing finitely many samples on the grid point. In our definition, the action sequence $a_t$ is a piecewise constant and is given by a sampling function at the grid points $a_{t_i}= \phi(t_i,X_{t_i},\xi_i)$. In contrast, for the grid sampling SDE in \cite{bender2024grid}, the action sequence is evolving as $a_t = \phi(t,X_t,\xi_{t_i})$, where the random source $\xi_{t_i}$ is piecewise constant. } in the subsequent analysis. 
For notational convenience, we write  \eqref{eq:sampled_sde} in the following equivalent form:
\begin{equation}\label{eq:sample_sde_abbre}
\dd X_t = b(t,X_t,a_{\delta(t)})\dd t + \sigma(t,X_t,a_{\delta(t)}) \dd W_t, \quad t\in [0,T];  \quad X_0 = x_0
\end{equation}
with  $\delta(t) \coloneqq t_i$ for $t\in [t_i,t_{i+1})$, and $a_{t_i}$ given in \eqref{eq:sampled_sde}. 
The dynamics \eqref{eq:sample_sde_abbre} can be viewed as an SDE with random coefficients. 

\begin{Remark}
    
Although the definition of the sampled dynamics depends on the sampling scheme $ \phi $ and $ \xi $, the weak convergence rate of the sampled dynamics to the aggregated dynamics depends only on the properties of the stochastic policy $ \bpi $, rather than the specific choice of $ \phi $ and $ \xi $ (Sections \ref{sec: weak_converge_both noise} and \ref{sec:conditional_weak_error}). The properties of $ \phi $ will be used to quantify the pathwise error between the sampled dynamics and the aggregated dynamics (Section \ref{sec:strong convergence}).

Note that the sampled dynamics evolve continuously over time while the control process remains constant within each subinterval. In particular, a random action $a_{t_i}$ is 
generated at $t_i$ and applied to the system over the interval $[t_i,t_{i+1})$ before being updated to the next action  $a_{t_{i+1}}$. This aligns with the RL setting, where the agent observes a continuous-time system at discrete time points and adjusts their actions based on these observations. 
\end{Remark}

The following lemma shows that the sampled dynamics      \eqref{eq:sample_sde_abbre} admits a unique strong solution $X^{\mathscr G}$
which is adapted to the 
filtration generated by both the Brownian motion $W$ and the execution noise $\xi$. 
The proof is given in Appendix \ref{sec proof: lemma}.

\begin{Lemma}\label{lemma:moments of X}
Suppose (H.\ref{assum:standing}) holds. Let $x_0\in \sR^d$,
$(\Omega,\mathcal F,\sF, \mathbb P)$
be the  probability space 
defined in \eqref{eq:space},
and 
$\sF=(\cF_t)_{t\ge 0}$ be the filtration such that 
$ \mathcal F_t \coloneqq  \sigma\{ (W_s)_{s\leq t}, (\xi_i)_{i=0}^\infty\}\vee \cN$, where $\cN$  \change{is the $ \sigma$-algebra of $\sP$-null sets}.  
Then for any grid $\mathscr G$, \eqref{eq:sample_sde_abbre} admits a unique
$\sF$-adapted 
strong solution $X^{\mathscr G}$
on the  space $(\Omega,\mathcal F, \mathbb P)$.
Moreover, for all $p\ge 2$, there exists a constant $C_p\ge 0$
such that for all grids $\mathscr G$ and all $t\in [0,T]$,
$  \E\left[|\X_t|^{p}\right] \le C_p(1+|x_0|^{p})e^{C_pt}$, and 
$\X_t$ is adapted to  $\mathcal G_t
\coloneqq 
  \sigma\{ (W_s)_{s\le t},  (\xi_{i})_ {t_i < t} \}\vee \cN$. 
\end{Lemma}

\section{Error Estimates for the Sampled Dynamics}
\label{sec:convergence sampled state}

In this section, we analyze the convergence rate of the sampled dynamics \eqref{eq:sample_sde_abbre} to the aggregated dynamics \eqref{eq:aggregated_sde} under various criteria as
the time grid $\mathscr G$ is refined. In particular, we quantify the weak error for quantities involving expectations over the system noise $W$ (but not necessarily over the execution noises $(\xi_i)_{i=0}^\infty$), and the pathwise error $\tilde X- X^{\mathscr{G}}$.
 To this end, 
we define the mesh size of a time grid $\mathscr G= \{0=t_0<\ldots <t_n =T\} $ by $|\mathscr{G}| = \max_{0\le i\le n-1} (t_{i+1} - t_{i})$.

\subsection{Weak error with expectations over execution noise}
\label{sec: weak_converge_both noise}

We start by analyzing the weak convergence rate of $X^{\mathscr{G}}$ to $\tilde X$ under the measure $\PP$.
That is, for a suitable test function $f:\sR^d\to \sR$,
we quantify the error $\E[f(\aX_s)]- \E[f(\X_s)]$  in terms of the time grid  $\mathscr{G}$, where the expectation is taken over both system noise $W$ and the execution noises $(\xi_i)_{i=0}^\infty$ (cf.~ \eqref{eq:space}).

\subsubsection{Weak convergence for regular coefficients}

\label{sec:weak_error_in_expectation}

To highlight the key ideas without introducing needless technicalities, we begin by providing the analysis for sufficiently smooth coefficients and test functions. In this setting, we obtain the maximum convergence rate  of the order $\mathcal O(|\mathscr{G}|)$ as $|\mathscr{G}|$ tends to zero.

To specify the precise regularity conditions, we introduce the following notation: for each Euclidean space $E$ and $k \in \sN$, $p\in\sN_0$, let $C^{k}_p([0,T]\times \sR^d; E)$ be the space of functions $u:[0,T]\times \sR^d\to E $ such that for all $r\in\sN_0$ and multi-indices $s$ satisfying $2r+|s| \le k$, the partial derivative $\partial_t^r\partial_x^s u$ exists and is continuous for all $(t,x)\in [0,T]\times\R^d$, and they all have polynomial growth in $x$:
\begin{equation}\label{eq:poly norm def}
\|u\|_{C_p^k} :=  \sum_{2r+|s|\le k} \sup_{(t,x)\in [0,T]\times\R^d} \frac{\left|\partial_t^r\partial_x^s u(t,x)\right|}{1+|x|^p}<\infty.
\end{equation}
Let $C^{k,0}_p([0,T]\times \sR^d\times A; E)$
be the space of functions $u: [0,T]\times \sR^d \times \A\to E$ such that for all $r\in \sN_0$ and multi-indices $s$ satisfying $2r+|s| \le k$, the partial derivative
$  \partial_t^r\partial_x^s u(t,x,a) $ exists, is continuous with respect to $t$ and $x$ for all $(t,x,a)\in [0,T]\times\R^d\times\A$ and they all have polynomial growth in $x$ and $a$:
\begin{equation}
\|u\|_{C_p^{k,0}} :=  \sum_{2r+|s|\le k} \sup_{(t,x,a)\in [0,T]\times\R^d\times A} \frac{\left|\partial_t^r\partial_x^s u(t,x,a)\right|}{1+|x|^p + d_A(a,a_0)^p}<\infty,
\end{equation}
where $a_0\in A$ is the fixed element in (H.\ref{assum:standing}).
For simplicity, we omit the range space $E$ when no confusion arises. 
Moreover, we denote by $\mathcal L$ the generator of the aggregated dynamics \eqref{eq:aggregated_sde} satisfying   
for all $\phi \in C^2([0,T]\times \sR^d)$,
\begin{equation}\label{eq: generator L}
\mathcal L \phi (t,x) = \tilde{b}^{\bpi} (t,x)\nabla_x \phi(t,x)+\frac{1}{2}\tr\left((\tilde{\sigma}^{\bpi}(\tilde{\sigma}^{\bpi})^\top)(t,x)     
\textrm{Hess}_x\phi(t,x)\right).
\end{equation}

The following assumption presents the regularity conditions to facilitate the analysis.

\begin{Assumption}\label{assump:solution_kol_pde}
\begin{enumerate}[(1)]
\item  $b$ and $   \sigma\sigma^{\top}$ are in the space   $ {C}^{2,0}_p([0,T]\times \sR^d\times A)$ for some $p\in \sN_0$;
\item 
$\tilde b^{\bpi}, \tilde \sigma^{\bpi}$ are sufficiently regular, such that for any $f\in C_p^4(\mathbb R^d)$ and any $   t^{\prime}\in [0,T]$, the PDE 
\begin{equation}\label{eq:Kol_pde}
(\partial_t v+\mathcal{L} v)(t,x)   = 0,\quad  (t,x)\in [0,t')\times \sR^d;\quad 
v(t^{\prime},x) = f(x), \quad x\in \sR^d
\end{equation}
has a unique solution $ v_f(\cdot,\cdot;t^{\prime}) \in C_p^{4}([0,t^{\prime}]\times \R^d)$.  Moreover, $ \|v_f(\cdot,\cdot;t^\prime)\|_{C_p^4} \leq C \|f \|_{C_p^4}$, where $C$ is a constant depending only on $T,\tilde b^{\bpi}, \tilde \sigma^{\bpi},p$.
\end{enumerate}

\end{Assumption}

\begin{Remark}\label{rek:conditions for kol pde} 
\change{
To obtain maximal first-order weak convergence, Condition (H.\ref{assump:solution_kol_pde}) requires that both the state coefficients and the solution to the associated Kolmogorov PDE possess sufficient regularity. This regularity requirement is stronger than  Condition (H.\ref{assum:lipschitz strong conv}) assumed   for half-order strong convergence. 
This distinction is consistent with the fact that establishing first-order weak convergence of Euler–Maruyama schemes for SDEs typically requires stronger regularity conditions than the Lipschitz continuity that ensures strong convergence (see, e.g., \cite[Theorem 14.1.5]{kloeden1992stochastic} and \cite{higham2002Strong} for the conditions of weak and strong convergence of Euler–type schemes, respectively). However, as emphasized in Section \ref{sec:introduction}, our convergence result does not follow from existing results on Euler–Maruyama schemes, since the sampled dynamics \eqref{eq:sample_sde_abbre} have random coefficients and evolve continuously, and consequently do not arise from a standard Euler–Maruyama discretization of the aggregated dynamics \eqref{eq:aggregated_sde}.

The  PDE solution regularity   in Condition (H.\ref{assump:solution_kol_pde})   holds when the coefficients of the aggregated dynamics \eqref{eq:aggregated_sde} are sufficiently regular.
For instance, if $\tilde b^{\bpi} $ and $ \tilde \sigma^{\bpi}$ are in the  class $ C_p^{4}$ and have bounded  derivatives, then
(H.\ref{assump:solution_kol_pde}) can be verified using the probabilistic argument in \cite{pardoux1992backward}.  Alternatively, if the diffusion matrix $\tilde \sigma^{\bpi}(\tilde \sigma^{\bpi})^\top$ is uniformly elliptic and $   \tilde b^{\bpi}, \tilde \sigma^{\bpi} (\tilde \sigma^{\bpi})^{\top} $ are bounded   functions in $C^3_p$ with bounded derivatives,  then  the desired solution regularity follows from  the standard    regularity theory of parabolic PDEs; see e.g., 
\cite[Lemma 2.2]{ma2024convergence} and 
\cite[Chapter 4, Theorem 5.1]{ladyzhenskaya1968linear}. 
The details are left to the interested reader.
 }
\end{Remark}

The following theorem establishes the optimal first-order weak convergence of the sampled dynamics to the aggregated dynamics.
\begin{Theorem}\label{thm:weak convergence}
Suppose (H.\ref{assum:standing}), (H.\ref{assum:Lipschitz}) and (H.\ref{assump:solution_kol_pde}) hold with $p\ge2$. Then for any $f\in C_p^4(\mathbb R^d)$, there exists a constant $C_p\ge 0$ depending only on $T,\tilde b^{\bpi}, \tilde \sigma^{\bpi},\bpi,b,\sigma,p$
such that for all grids $\mathscr{G}$,
\begin{equation}
\sup_{t\in [0,T]}\left| \E[f(\X_{t })] - \E[f(\aX_{t})]\right| \le C_p\|f\|_{C_p^4}|\mathscr{G}|.
\end{equation}
\end{Theorem}
\begin{proof}
See Section \ref{sec proof: thm weak convergence}.
\end{proof}

By  Theorem \ref{thm:weak convergence},
for a given   function $f\in C_p^4(\mathbb R^d)$,
$f(\X_T)$ is a biased estimator of $\E[f(\aX_T)]$, and the bias vanishes at first order, 
  as the sampling mesh size 
$|\mathscr{G}|$ tends to zero.
  \change{Moreover, by Lemma \ref{lemma:moments of X}, the variance of the estimator $f(\X_T)$ is uniformly bounded, independent of the mesh size $|\mathscr{G}|$. The computational cost of generating one trajectory of $\X_T$ increases linearly with the number of grid points in $\mathscr{G}$.
  For further details on its implications for the sample complexity of Monte Carlo estimators of $\E[f(\aX_t)]$, we refer to Remark \ref{rmk:sample_complexity}.
   
  }

We illustrate several common cases in which the required assumptions in Theorem \ref{thm:weak convergence} hold. 
\begin{Example}[General smooth coefficients]
\label{example:general case}
Assume $\bpi(\dd a|t,x)$ admits a  density function with respect to a  measure $ \mu$ on $A$ (e.g., Lebesgue measure), i.e.,
\begin{equation}\label{eq:example general density}
\bpi(\dd a| t,x)  = H(t,x,a)\mu (\dd a).
\end{equation}
 If it holds
 for any  integer $ r \ge 0$ and multi-index $s$ satisfying $2r + |s| \le 4$ that 
\begin{enumerate}[(\alph*)]
\item  For $ \phi \in \{b,\sigma\}$, $ (t,x,a)\mapsto \partial_t^r \partial_x^s \phi(t,x,a)$ is bounded and continuous.
\item  $ (t,x)\mapsto 
 \int_\A |\partial_t^r\partial_x^s H(t,x,a)| \mu(\dd a)$ is bounded.
\item  $ \sigma\sigma^{\top}$ is uniformly elliptic.
\end{enumerate}
Then $ \tilde b^{\bpi}$ and $ \tilde \sigma^{\bpi}$ are of class $ C_p^{4}$ with bounded derivatives.
\end{Example}

\begin{Example}[Gaussian policies]
\label{example:guassian}
Assume $ A = \R^m$ and $ \bpi(\cdot|t,x)$ is Gaussian with respect to the Lebesgue measure on $ \R^m$, i.e.,
\begin{equation}\label{eq:example gaussian denstiy}
 {\bpi(\dd a|t,x)}  = \varphi(a|\mu(t,x),\Sigma^2(t,x))\dd a,
\end{equation}
where $ \varphi(\cdot|\mu,\Sigma^2)$ is the Gaussian density with mean $ \mu$ and variance $ \Sigma^2$. 
If the following conditions hold: 
\begin{itemize}
\item [(1)] For $\phi \in \{b,\sigma\}$, integer $r\ge 0 $ and multi-index $s_1, s_2$ satisfying that $0< 2r+ |s_1| + |s_2| \le 4$, $ 
(t,x,a)\mapsto 
\partial_t^l \partial_x^{s_1} \partial_a^{s_2} \phi(t,x,a)$ is bounded and continuous.
\item [(2)] $  \mu,\Sigma \in C_p^4$ with bounded derivatives. 
\item [(3)] $ \sigma\sigma^{\top}$ is uniformly elliptic.
\end{itemize} 
Then $ \tilde b^{\bpi}$ and $ \tilde \sigma^{\bpi}$ are of class $ C_p^{4}$ with bounded derivatives.

\end{Example}

Examples \ref{example:general case} and \ref{example:guassian} assume that \( \sigma \sigma^{\top} \) is uniformly elliptic, which implies the uniform  ellipticity of 
 $(t,x)\mapsto \int_A \sigma\sigma^{\top}(t,x,a) \bpi(\dd a|t,x) $. 
This nondegeneracy condition ensures that after taking the square root, 
\( (t,x)\mapsto \tilde{\sigma}^{\bpi} 
(t,x)
\) 
has the same regularity as 
$(t,x)\mapsto \int_A \sigma\sigma^{\top}(t,x,a) \bpi(\dd a|t,x) $. 
The existence of a sufficiently regular square root $\tilde{\sigma}^{\bpi}$ can be ensured by replacing the nondegeneracy condition with certain structural conditions on $\sigma$ and $\bpi$. Specifically, if we can express \( \int_A \sigma\sigma^{\top}(t,x,a) \bpi(\dd a|t,x)  = \lambda(t,x)\lambda^{\top}(t,x) \) with a sufficiently smooth function \( \lambda(t,x) \), then \( \tilde{\sigma}^{\bpi} = \lambda \) will have the desired regularity, even if it is degenerate. We provide several examples to illustrate this.

\begin{Example}[Uncontrolled volatility]
\label{example:degenerate}
Consider the case in which $ \sigma$ is uncontrolled, i.e., $ \sigma(t,x,a) = \sigma(t,x)$. In this case,
\begin{equation}
\int_A \sigma\sigma^{\top}(t,x,a)\bpi(\dd a|t,x) = \sigma\sigma^{\top}(t,x).
\end{equation}
Note that we only require $ \tilde \sigma^{\bpi} (\tilde \sigma^{\bpi})^\top = \int_A \sigma\sigma^{\top}(t,x,a)\bpi(\dd a|t,x)$, which means that we can choose $ \tilde\sigma^{\bpi}$ to be $\sigma$ instead of $ \sqrt{\sigma\sigma^{\top}}$. Therefore, the conditions on $ \sigma $ in Examples \ref{example:general case} and \ref{example:guassian} can be relaxed so that $ \sigma$ is of class $C_p^{4}$ with bounded derivatives and it can be degenerate.
\end{Example}

\begin{Example}[Gaussian policies with linear volatility]
\label{eg:gaussian linear}
Consider the case where the volatility linearly depends on the action, i.e., $ \sigma(t,x,a) = a\beta^{\top}$ with $ \beta \in R^d$ and \eqref{eq:example gaussian denstiy} holds with $ m = d $ and $ \mu(t,x) = 0$ and, then 
\begin{equation}
\int_{A} \sigma\sigma^{\top}(t,x,a)\bpi(\dd a|t,x)= \int_{\R^d} \Sigma(t,x)u\beta^{\top}\beta u^{\top} \Sigma(t,x)^{\top} \varphi(u)\dd u.
\end{equation}
We can choose $ \tilde \sigma^{\bpi}(t,x)  = \Sigma(t,x)\sqrt{\int_{\R^d}u\beta^{\top}\beta u^{\top}\varphi(u)\dd u}$. Therefore, $ \tilde \sigma^{\bpi}$ can degenerate; the assumption in this case is that $ \Sigma \in C_p^{4}$ has bounded derivatives.

If $ d = 1$, then we can extend the second case to a general form. Assume that $\bpi(\cdot|t,x)$ is the location scaled distribution, i.e., there is a function $ \Pi(t,x)$ such that $ \bpi(\cdot|t,x)$ satisfies for any measurable function $ f$, $ \int_A f(a)\bpi(\dd a|t,x) = \int_B f(\Pi(t,x)u)\mu(\dd u)$, where $ (B,\mathscr{F}_B,\mu)$ is a measure space. Let $ \sigma(t,x,a) = h(t,x)a$ for a bounded $ h$. Moreover, assume that $ h \ge 0$ and $ \Pi \ge 0$. We have
\begin{equation}
(\tilde\sigma^{\bpi}(t,x))^2 = \int_B h^2(t,x)\Pi^2(t,x)u^2\mu(\dd u).
\end{equation}
If we can choose $ \tilde \sigma^{\bpi}(t,x) = h(t,x)\Pi(t,x)\sqrt{\int_B u^2 \mu(\dd u )}$, then $\tilde \sigma^{\bpi}$ can   degenerate as long as  $ h\Pi \in C^{4}_p$ and has bounded derivatives. 
\end{Example}

\subsubsection{Weak convergence with irregular coefficients}\label{sec:relax regularity pde approach}
In the previous section, we established a first-order weak convergence of the sampled dynamics under the assumption that 
all coefficients are sufficiently regular so that  
the PDE associated with the aggregated dynamics has a four-times differentiable solution (see (H.\ref{assump:solution_kol_pde}) and Remark \ref{rek:conditions for kol pde}).   In this section, we relax the regularity conditions on the coefficients and quantify the weak convergence of the sampled dynamics based on their smoothness.

 The following   H\"{o}lder spaces 
 will be used in the analysis.

\begin{Definition}
For $l\in (0,\infty)\setminus \N$,
let $\mathcal{H}_T^{(l)}$ be the space of functions $u:  [0,T]\times \R^d\times\A\to \change{E}$ such that  
 for all integer $r \ge 0$ and multi-index $s$ satisfying $2r + |s|\le l$,
the partial derivatives $\partial_t^r \partial_x^s u$ are bounded and H\"older continuous with index $l - [l]$ in $x$   and with index $(l - 2r - |s|)/2$ in $t$, i.e., 
\begin{equation}\label{eq:derivative poly growth}
\begin{aligned}
\|u\|_T^{(l)} :=& \sum_{2r + |s| \le [l]}\sup_{t,x,a}|\partial_t^r \partial_x^s u(t,x,a)| \\
&+\sum_{2r + |s| = [l]}  \sup_{t,a,x\ne x^{\prime}} \frac{\left| \partial_t^r \partial_x^s u(t,x,a) - \partial_t^r \partial_x^s u(t,x^{\prime},a)\right|}{\left| x - x^{\prime}\right|^{l- [l]}} \\
&+\sum_{0<l-2r-|s|<2}  \sup_{t\ne t^{\prime},x,a}\frac{\left| \partial_t^r \partial_x^s u(t,x,a) - \partial_t^r \partial_x^s u(t^{\prime},x,a)\right|}{\left| t - t^{\prime}\right|^{(l-2r-|s|)/2}} < +\infty.
\end{aligned}
\end{equation}
Let $\mathcal{H}^{(l)}$  be the space of all functions $u\in \mathcal{H}_T^{(l)}$ that do not depend on $t$ and $a$, 
and let 
$\|\cdot\|^{(l)}$
be the corresponding norm.
\end{Definition}

The following regularity condition is imposed in this section.

\begin{Assumption}\label{assump:solution Kol pde relax regularity}
Let $l \in (0,1)\cup(1,2)\cup(2,3)$, 
$b, \sigma\sigma^{\top} \in \mathcal{H}_T^{(l)} $,
$ \tilde b^{\bpi}, \tilde \sigma^{\bpi} (\tilde \sigma^{\bpi})^\top \in \mathcal{H}_T^{(l)}$, and  $\tilde\sigma^{\bpi}(\tilde\sigma^{\bpi})^\top$ is uniformly elliptic.
\end{Assumption}

Under   (H.\ref{assump:solution Kol pde relax regularity}), standard  regularity theory 
for parabolic PDEs (see, e.g.,
\cite[Chapter 4, Theorem 5.1]{ladyzhenskaya1968linear})
shows that   for any $f\in \mathcal{H}^{(l+2)}$ and $t'\in [0,T]$, the PDE \eqref{eq:Kol_pde} has a unique solution $ v_f(\cdot,\cdot;t') \in \mathcal{H}_{t'}^{(l+2)}$ and $ \|v_f(\cdot,\cdot;t')\|_{t'}^{(l+2)} \le C\|f\|^{(l+2)}$, where
the  constant
$ C$    depends only on $T, \tilde b^{\bpi}$ and $ \tilde \sigma^{\bpi}$. 

Leveraging this PDE regularity, the following theorem quantifies the weak convergence order of the sampled dynamics in terms of the smoothness coefficient $l$.

\begin{Theorem}\label{thm:weak convergence relax regularity}
 Suppose that (H.\ref{assum:standing}), (H.\ref{assum:Lipschitz}) and (H.\ref{assump:solution Kol pde relax regularity}) hold. Then  for $f\in \mathcal{H}^{(l+2)}$, there exists a constant $C$ depending only on $b,\sigma,\tilde b^{\bpi},\tilde \sigma^{\bpi},T$
such that for all grids $\mathscr G$,
\begin{equation}
\sup_{t\in [0,T]}\left| \E\left[ f(\X_t)\right] - \E\left[ f(\aX_t)\right]\right| \le C\|f\|^{(l+2)} |\mathscr{G}|^{\mathcal{X}(l)},
\end{equation}
where the constant $\mathcal{X}(l)$ is defined by 
\begin{equation}\label{eq:convergence rate}
\mathcal{X}(l) \coloneqq  
\left\{\begin{aligned}
&l/2, &l \in (0,1),\\
&1/(3-l), &l \in (1,2),\\
&1, &l\in (2,3).
\end{aligned}\right.
\end{equation}
\end{Theorem}

\begin{proof}
See Section \ref{sec proof: thm weak convergence relax regularity}.
\end{proof}

Theorem \ref{thm:weak convergence relax regularity} establishes the weak convergence order of the sampled dynamics while relaxing the regularity conditions on the coefficients in Theorem \ref{thm:weak convergence} and allowing for H\"older continuous coefficients. 
For clarity, we assume that the coefficients are bounded and that the volatility coefficients are uniformly elliptic, enabling us to leverage parabolic regularity theory to ensure the associated PDE admits a sufficiently regular classical solution. The result naturally extends to coefficients with polynomial growth and degenerate volatility, provided that the PDE \eqref{eq:Kol_pde} admits a classical solution with polynomial growth.

Here we emphasize that the analysis of policy execution error in Theorem \ref{thm:weak convergence relax regularity} fundamentally differs from the convergence rate analysis of the classical piecewise constant control approximation in \cite{jakobsen2019improved}. This distinction arises because the aggregated dynamics and the sampled dynamics have different coefficients and are thus associated with different PDEs—one governed by aggregated coefficients and the other by random coefficients influenced by execution noise. As a result, the viscosity solution arguments used in \cite{jakobsen2019improved} are not directly applicable when the PDE \eqref{eq:Kol_pde} fails to admit a classical solution.

\subsection{
Weak error conditioned on execution noise}
\label{sec:conditional_weak_error}

In this section, we establish high-probability bounds for the policy execution error with respect to the randomness used to sample the actions.

More precisely,
let $\cF^{\xi} := \sigma\{ \xi_{i} \mid {i \ge 0}\}$,
and consider the error term 
$\Ew[ f(\X_{t})] - \Ew[ f(\aX_{t})]$,
where $f$ is a given test function, and 
$\Ew [\cdot] $ denotes  the conditional expectation $\E[\cdot \mid \cF^{\xi}]$.
Note that the error $\Ew[ f(\X_{t})] - \Ew[ f(\aX_{t})]$ is a random variable since it depends on the realization of the execution noises   $\xi$. We have shown 
in Section \ref{sec: weak_converge_both noise} that the error  converges to zero 
in  expectation as $|\mathscr{G}|\to 0$. 
In the sequel, we provide a high-probability bound of this error.

\begin{Assumption}\label{assump:linear growth on x}
For every $ a \in A$, the functions $b(\cdot,\cdot,a)$ and $\sigma\sigma^{\top}(\cdot,\cdot,a)$ are of class $C_p^{2}([0,T]\times\R^d)$ for some $p \ge 2 $. Moreover, for $\varphi_1 \in \{b,\sigma\}$, there is a constant $C\ge 0$  such that
\begin{equation}\label{eq:b upper bound uniform a}
\left| \varphi_1(t,x,a)\right| \le C\left( 1+|x|\right), \quad \forall (t,x,a)\in[0,T]\times\R^d\times\A,
\end{equation}
and $\sup_{a\in A}\|b(\cdot,\cdot,a)\|_{C_p^2}<\infty$
and $\sup_{a\in A}\|\sigma\sigma^{\top}(\cdot,\cdot,a)\|_{C_p^2}<\infty$.
\end{Assumption}

\begin{Remark}
Compared with (H.\ref{assump:solution_kol_pde}),
(H.\ref{assump:linear growth on x}) imposes the additional  condition on the uniform boundedness  
$\sup_{a\in A}\|b(\cdot,\cdot,a)\|_{C_p^2}<\infty$, and $\sup_{a\in A} \|\sigma\sigma^{\top}(\cdot,\cdot,a)\|_{C_p^2}<\infty$,
which allows for controlling the concentration behavior of the sampled dynamics with respect to the execution noise $\xi$.
 It is easy to see that (H.\ref{assump:solution Kol pde relax regularity}) includes the uniform boundedness assumption on $a$, thus, if (H.\ref{assump:solution Kol pde relax regularity}) holds, then we do not need (H.\ref{assump:linear growth on x}).
\end{Remark}

The following theorem establishes that the sampled dynamics converge weakly at half-order, averaged over the system noise, uniformly over the realizations of the sampling noises with high probability.
It extends   \cite[Theorem 2.2]{szpruch2024} established for LQ problems with Gaussian policies to general state dynamics and stochastic policies.

\begin{Theorem}\label{thm:concentration inequality}
Suppose (H.\ref{assum:standing}), (H.\ref{assum:Lipschitz}), (H.\ref{assump:solution_kol_pde}) and (H.\ref{assump:linear growth on x}) hold. Then for 
all $f\in C^4_p(\sR^d)$,
there is a constant $C_p$ depending only on $b,\sigma,\tilde b^{\bpi},\tilde \sigma^{\bpi},T,p$, such that  
for all grids $\mathscr G$ and all $\varrho \in (0,1) $,
\begin{equation}\label{eq: concentration ineq}
\sup_{t\in [0,T]}\PP\left[ \left| \Ew[f(\X_{t})] - \E[f(\aX_{t})]\right| > C_p\|f\|_{C_p^4}\sqrt{|\mathscr{G}|}\big(1+\sqrt{\ln (2/\varrho)}\big)\right] \le \varrho.
\end{equation}

 Suppose   instead  that (H.\ref{assum:standing}), (H.\ref{assum:Lipschitz}) and (H.\ref{assump:solution Kol pde relax regularity}) hold.
Then for  
all $f\in \mathcal{H}^{(l+2)}$
with $l$ in (H.\ref{assump:solution Kol pde relax regularity}),
there is a constant $C$ depending only on $b,\sigma,\tilde b^{\bpi},\tilde \sigma^{\bpi},T$ such that for all grids $\mathscr G$ and all $\varrho \in (0,1) $, 
\begin{equation}\label{eq: concentration ineq relax regularity}
\sup_{t\in [0,T]} \PP\left[ \left| \Ew[f(\X_{t})] - \E[f(\aX_{t})]\right| > C\|f\|^{(l+2)}|\mathscr{G}|^{\mathcal{X}(l)\land \frac{1}{2}}\big(1+\sqrt{\ln (2/\varrho)}\big)\right] \le \varrho,
\end{equation}
with $\mathcal{X}(l)$ defined by \eqref{eq:convergence rate}.
\end{Theorem}

The proof of Theorem \ref{thm:concentration inequality}, provided in Section \ref{sec proof: thm conentration inequality}, is based on establishing a pointwise upper bound for the weak error on each subinterval
in terms of the grid size,
and subsequently controlling the overall error using concentration inequalities for bounded martingale difference sequences.

\change{
Note that Theorem \ref{thm:concentration inequality} shows that the convergence of the sampled dynamics to the aggregated dynamics can be achieved simply by increasing the sampling frequency, without the need to average over multiple realizations of the sampling noise at fixed time points. 
This observation significantly reduces the sample complexity of estimators for the aggregated dynamics.

To see it, 
consider estimating 
$\mathbb E[f(\aX_{T})]$ using the sampled dynamics \eqref{eq:sampled_sde}.
Let  $\mathscr{G}_n=\{i\frac{T}{n}\}_{i=0}^n$, $n\in \sN$,
be a uniform time grid,
 and let  $(W_k)_{k\in \sN}$ 
 be a collection of independent $d$-dimensional Brownian motions, which are also independent of $\xi$. 
 For each $m \in \sN$,   define the   estimator  
 $\hat I_{m,n}$ of    $\mathbb E[f(\aX_{T})]$ by
\begin{equation}
\label{eq:mc_estimator_f(X_T)}
\hat I_{m,n}\coloneqq 
\frac{1}{m}\sum_{k = 1}^mf(X^{\mathscr{G}_n}_T(W_k, \xi)),
\end{equation}
where  $X^{\mathscr{G}_n}_{T}(W_k,\xi)$ denotes   the terminal state of the sampled dynamics \eqref{eq:sampled_sde}, driven by the Brownian motion $W_k$
 and the common sampling noises 
$\xi$ on the grid 
$\mathscr{G}_n$. 
 We define  the complexity of $\hat I_{m,n}$ as the total number of random variables,  including   the Brownian motions $(W_k)_{k\in \sN}$ and the sampling noises $(\xi_n)_{n\in \sN_0}$,    required to evaluate a single realization of  $\hat I_{m,n}$.
 
\begin{Corollary}\label{cor:concentration inequality}
 Suppose (H.\ref{assum:standing}), (H.\ref{assum:Lipschitz}), (H.\ref{assump:solution_kol_pde}) and (H.\ref{assump:linear growth on x}) hold.   
Let $f\in C^4_p(\sR^d)$ be bounded.   There is a constant $C\ge 0$   such that 
for all 
$\varepsilon\in (0,1)$,
by setting $m_\varepsilon=n_\varepsilon=  \lceil \varepsilon^{-2} \left(1+\ln(2/\varepsilon)\right) \rceil$,
it holds with probability at least $1-\varepsilon$ that
$|\hat I_{m_\varepsilon,n_\varepsilon}- \E[f(\aX_T)]|\le C\varepsilon$.
Moreover, the  
  sample complexity of $\hat I_{m_\varepsilon,n_\varepsilon}$ is $2\lceil \varepsilon^{-2} \left(1+\ln(2/\varepsilon)\right) \rceil$.

\end{Corollary}

The proof of Corollary \ref{cor:concentration inequality} is given in Section \ref{sec proof: cor concen ineq}. Similar results hold when the objective involves a sufficiently regular running reward  of the aggregated process, as in the cases analyzed in Section \ref{sec:applications}.

\begin{Remark}
    \label{rmk:sample_complexity}
Note that in the idealized scenario where one can directly sample trajectories of the aggregated dynamics, a Monte Carlo estimator of $\E[f(\aX_T)]$ achieves a sample complexity of order $\cO(\varepsilon^{-2})$. As shown in Corollary \ref{cor:concentration inequality}, the estimator \eqref{eq:mc_estimator_f(X_T)} attains this optimal sample complexity up to a logarithmic factor. This is due to  the fact that the estimator \eqref{eq:mc_estimator_f(X_T)} does not  regenerate sampling noises for different Brownian motion trajectories.

In contrast, the 
     weak convergence results   in Theorems \ref{thm:weak convergence} and \ref{thm:weak convergence relax regularity} suggest considering the naive Monte Carlo estimator 
     \begin{equation}
     \label{eq:mc_estimator_f(X_T)_2}
\tilde I_{m,n}\coloneqq 
\frac{1}{m}\sum_{k = 1}^mf(X^{\mathscr{G}_n}_T(W_k, \xi_k)),
\end{equation}
where $(\xi_k)_{k\in \sN}$ are independent copies of the sampling  noises  $\xi$. Under the same conditions as in Corollary \ref{cor:concentration inequality}, 
the standard bias–variance decomposition shows that for all $m,n\in \sN$,
\begin{align*}
    \sE[|\tilde I_{m,n}-\E[f(\aX_T)]|^2]^{1/2}\le C\left(\frac{1}{\sqrt m}+\frac{1}{n}
    \right),
\end{align*}
for a constant $C\ge 0$ independent of $m$ and $ n $. To achieve a total error $\varepsilon$, we choose $m_\varepsilon=\cO(\varepsilon^{-2})$ and $n_\varepsilon =\cO(\varepsilon^{-1})$. However, since the estimator \eqref{eq:mc_estimator_f(X_T)_2} requires regenerating sampling noises for each trajectory, the overall complexity is $\cO(m_\varepsilon n_\varepsilon) = \cO(\varepsilon^{-3})$, which is significantly   higher than that of the estimator \eqref{eq:mc_estimator_f(X_T)}.

\end{Remark}
 }

\subsection{Strong convergence}
\label{sec:strong convergence}

In this subsection, we quantify the pathwise difference between the sampled dynamics and the aggregated dynamics in terms of the sampling frequency for a fixed realization of the Brownian motion.  
To this end, we assume the volatility $ \sigma$ is uncontrolled, i.e., $ \sigma(t,x,a) = \sigma(t,x)$, and impose regularity conditions on the function \( \phi \) used for sampling the actions (recall that 
$a_{t_i}=\phi(t_i,\X_{t_i},\xi_i)$ as in \eqref{eq:sampled_sde}).

\begin{Assumption}\label{assum:lipschitz strong conv}
\begin{enumerate}[(1)]
\item 
\label{item:lipschitz_phi}
{There exists   $ C\ge 0$  such that for all $ x,y \in \R^{d}$, $a_1, a_2 \in A$}, and   $  t\in[0,T]$,
\begin{align*}
\left| b(t,x,a_1) - b(t,y,a_2)\right| +\left| \sigma(t,x) - \sigma(t,y)\right|&\le C\left(\left| x-y\right|+ d_A( a_1,a_2)\right),
\\
\left| b(t,0,0) \right| +\left| \sigma(t,0)  \right|&\le C.
\end{align*}

\item
\label{item:subexponential}
There exists $(t_0,x_0,a_0) \in [0,T]\times \R^d \times A$, such that $ d_A(\phi(t_0,x_0,\xi),a_0)$ is sub-exponential.\footnote{We say a random variable $X$ is sub-exponential if there exists a constant $C > 0$ such that $\E[|X| > t] \leq 2 e^{- C t}$ for every $t\geq 0$.} 
Moreover, 
there exists   $ C\ge 0$  such that for all $ (t_1,x_1), (t_2,x_2) \in [0,T]\times\R^{d}$, and $ u\in E$, 
\begin{equation}
\label{eq:phi_lipschitz}
 d_A( \phi(t_1,x_1,u), \phi(t_2,x_2,u)) \le C\left( \left| t_1-t_2\right|^{\frac{1}{2}} + \left| x_1 - x_2\right| \right).
\end{equation}
\end{enumerate}
\end{Assumption}

 Assumption (H.\ref{assum:lipschitz strong conv}\ref{item:subexponential}) 
holds for commonly used Gaussian policies considered in \cite{szpruch2024} and many others.

\begin{Example}
\label{eg:normal distribution for strong convergence}
Let  $A=\sR^k$, and $\bpi(\cdot|t,x)$ be a Gaussian policy of the form (see e.g.,  \cite{szpruch2024}):
\begin{equation}
 {\bpi(\dd a|t,x)}  = \varphi(a|\mu(t,x),\Sigma^2)\dd a,
\end{equation}
where $ \varphi(\cdot|\mu,\Sigma^2)$ denotes the Gaussian density with mean $ \mu$ and variance $ \Sigma^2$. 
In this case, 
one can define the sampling function  $ \phi(t,x,\xi) = \mu(t,x) + \Sigma \xi$, with   a standard normal random variable  $\xi$. If $\mu$
is Lipschitz continuous in $x$ and  $1/2$-H\"older continuous in $t$,
then   $\phi$ satisfies   (H.\ref{assum:lipschitz strong conv}\ref{item:subexponential}). 
 
\end{Example}

\change{
More generally,
Assumption (H.\ref{assum:lipschitz strong conv}\ref{item:subexponential}) 
holds for stochastic policies with sufficiently regular densities. 
 
\begin{Example}
   Let  $A=[\alpha,\beta]$ for   $\alpha<\beta$, and  
   $  \bpi(\d a |t,x) =h(t,x,a)\d a  
   $ 
   with a continuous   density   $h:[0,T]\times \sR^d\times A\to \sR$. In this case,
   by the inverse transform method,
   one can sample from $\bpi(\d a|t,x)$
   by $\phi(t,x,\xi)$,
   where 
   $\xi$  is 
   a uniform random variable on $[0,1]$, and   $[0,1]\ni u\mapsto \phi(t,x,u)\in A$ is the inverse  function 
   of the map  
$A\ni r\mapsto H(t,x,r)\coloneqq \int_\alpha^r  h(t,x, a) \d a\in [0,1]$.
If there exists $C,c>0$ such that for all $(t,x), (t',x')\in [0,T]\times \sR^d$ and $a\in A$, 
\begin{equation}
\label{eq:lipshcitz_density}
h(t,x,a)\ge c, 
\quad 
|h(t,x,a)-h(t',x',a)|
\le C(|t-t'|^{\frac{1}{2}}+
|x-x'|),    
\end{equation}
then $\phi$
  satisfies   (H.\ref{assum:lipschitz strong conv}\ref{item:subexponential}).

 Indeed,  
 $|\phi(t_0,x_0,\xi)|$ is sub-exponential  since $A$ is bounded. 
Moreover, 
since    $\phi$ is the inverse of $H$,
\begin{align*}
&H(t,x,\phi(t,x,u))
-
H(t,x,\phi(t',x',u))
+H(t,x,\phi(t',x',u))-
H(t',x',\phi(t',x',u))
\\
&=
H(t,x,\phi(t,x,u))
-H(t',x',\phi(t',x',u))=u-u=0, 
\end{align*}
which along with the mean value theorem and the definition of $H$ yields that  
\begin{align*}
& \phi(t,x,u)
-
\phi(t',x',u) 
= \frac{\int_\alpha^{\phi(t',x',u)}  (h(t' ,x' ,a) -
h(t ,x,a))\d a }
{\partial_r H (t,x,\tilde a) } 
\le 
\frac{\int_\alpha^{\beta}  |h(t' ,x' ,a) -
h(t ,x,a)|\d a }
{h(t,x,\tilde a) } 
\end{align*}
for some $\tilde a$
between
$\phi(t,x,u)
$ and $
\phi(t',x',u)$.
Using  \eqref{eq:lipshcitz_density}
shows that 
$\phi$ satisfies \eqref{eq:phi_lipschitz}.  

Similar arguments extend to multidimensional action sets  $A$
using the conditional inverse sampling method (see e.g., \cite[Chapter 6.3]{cherubini2004copula}). 

\end{Example}

}

Now we prove the strong convergence of the sampled dynamics under Assumption (H.\ref{assum:lipschitz strong conv}). \change{ Note that (H.\ref{assum:lipschitz strong conv}) implies (H.\ref{assum:standing}) and (H.\ref{assum:Lipschitz}), which ensure that the aggregated state process 
$\aX$ and 
the sampled state process
$\X$ are $p$-th power integrable for all $p \ge 2$ (see  Lemma \ref{lemma:monments strong conv}).}

\begin{Theorem}\label{thm:strong conv}
Suppose  (H.\ref{assum:lipschitz strong conv}) holds. For all $ p\ge 2$ and   all grids $\mathscr{G}=\{0=t_0<\ldots <t_n=T\}$ 
with $|\mathscr{G}|^{1/2}\log(1+n)\le C_G<\infty$,
\begin{equation}\label{eq:strong conv}
\E\left[ \sup_{0\le t\le T}\left| \X_t - \aX_t\right|^p\right] \le C|\mathscr{G}|^{\frac{p}{2}},
\end{equation}
where the   constant $C$
depends only on $b,\sigma,\phi,T,p$ and $C_G$. 

Consequently, if $\mathscr{G}=\{t_i\}_{i=0}^n$ with   $t_i = i\frac{T}{n}$, then for every $\varepsilon > 0$ and almost every $\omega \in \Omega$, there is an integer $N_{\varepsilon}(\omega)$ such that for every $n \ge N_{\varepsilon}(\omega)$, 
\begin{equation}
\sup_{0\le t\le T}\left| \X_t(\omega) - \aX_t(\omega)\right| \le n^{-\frac{1}{2}+\varepsilon}.
\end{equation}
\end{Theorem}
\begin{proof}
See Section \ref{sec proof: thm strong conv}.
\end{proof}

The proof of Theorem \ref{thm:strong conv}
is given in Section \ref{sec proof: thm strong conv}.
The analysis involves carefully estimating the Orlicz norm of the pathwise difference between the aggregated and sampled dynamics on each subinterval by conditioning on the realization of the Brownian motion, and then controlling the overall error using concentration inequalities for conditional sub-exponential martingale difference sequences.

Theorem \ref{thm:strong conv} indicates  that the convergence of the sampled dynamics to the aggregated dynamics can be interpreted as   \emph{a law of large numbers in time}.
Specifically,  
   when the volatility is uncontrolled, the $1/2$-order convergence is achieved by increasing the sampling frequency for any fixed  realization of the system noise.
  
 However,
 such a pathwise convergence fails when the volatility is controlled, as shown in the following example.

\begin{Example}
\label{eg:counter example}
Let 
$d=1$, $A\subset \sR$,
$ b(t,x,a) = 0$ and $ \sigma(t,x,a) = a$.
Assume $ \bpi(\dd a|t,x) = \mu(\dd a)$ 
for some $\mu\in \cP(A)$
and $ Q := \int_A a^2\mu(\dd a) < +\infty$. Assume that $ \{\xi_i\}_{i=0}^\infty$ are i.i.d.~random variables with distribution $\mu$. Let $a_{t_i} = \xi_i$  for all $i$, then we have 
\begin{equation}
\X_t - \aX_t = \int_0^t \left(a_{\delta(s)} - \sqrt{Q}\right) \dd W_s.
\end{equation}
Note that $ W$ is a Brownian motion under the filtration 
$\{ \mathcal{F}_t\}_{t\ge 0}$
with 
$ \mathcal{F}_t := \sigma\{a_{\delta(s)}:s\le t\}\lor \mathcal{F}^W_t$. By  It\^o's isometry and Fubini's theorem, 
\begin{equation}
\E\left[ |\X_t-\aX_t|^2\right] = \E\left[\int_0^t (a_{\delta(s)} - \sqrt{Q})^2 \dd s \right] = t\E\left[ (\xi_1 - \sqrt{Q})^2\right], 
\end{equation}
which does not converge  to zero   as $ |\mathscr{G}| \rightarrow 0$, except that $\xi_1$ is a constant (or equivalently $\mu$ is a Dirac distribution).
\end{Example}

\change{
\subsection{Numerical illustration}
 We present numerical results illustrating the theoretical strong and weak convergence rates for two representative one-dimensional cases: uncontrolled volatility $\dd X_t = a_t\dd t+ \dd W_t$ and controlled volatility $\dd X_t = a_t\dd W_t$. In both cases, we 
 set 
  the time horizon $[0,1]$ and 
the initial state $X_0=0$, and 
 consider sampling actions from a    standard  Gaussian policy (i.e., $a_t\sim  \mathcal{N}(0,1)$) over   a uniform   grid $\mathscr{G}=\{i\frac{T}{n}\}_{i=0}^n$, where $n\in \sN$ is the number of grid points.  
We consider the quantity 
$ |\E[(\X_T)^4] - \E[(\aX_T)^4]|$ to examine the 
 weak convergence of the sampled dynamics,
 and  
 the root mean squared error (RMSE)
 $ \sqrt{\E[|\X_T - \aX_T|^2]}$ to assess  
  the strong convergence. All  expectations
  are estimated using  the Monte Carlo method with 50000 samples. Each point   in the following plots is an average over 100 independent runs, and the shaded regions represent the 95\% confidence intervals, estimated using 1000 bootstrap resamples.

\begin{figure}[!ht]
    \centering
    \includegraphics[width=\textwidth]{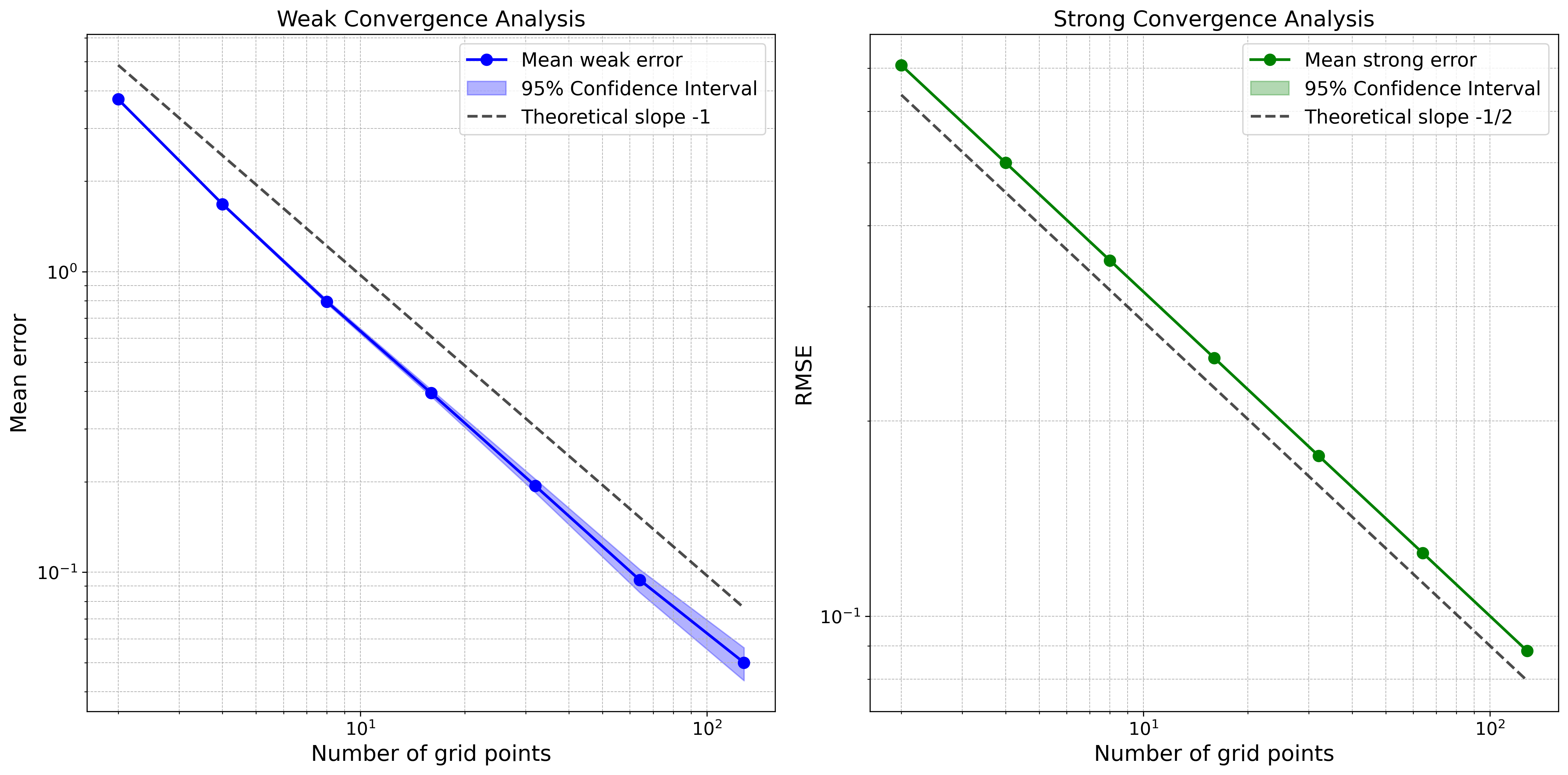}
    \caption{Weak and strong convergence analysis for the uncontrolled volatility case $\dd X_t = a_t\dd t+ \dd W_t$ with actions $a_t\sim \mathcal{N}(0,1)$. The corresponding aggregated dynamics is $\dd \tilde X_t = \dd W_t$. The test function is $f(x)=x^4$ and $T=1$. Left: Weak error versus the number of grid points $n$. Right: RMSE versus the number of grid points $n$. Both axes are on log scales.}
    \label{fig:drift_control}
\end{figure}

Figure \ref{fig:drift_control} presents the convergence results for the case with uncontrolled volatility. The left panel shows that the   weak error decays linearly with  the sampling frequency  on a log-log scale, with a slope of approximately $-1$. This  confirms  the first-order convergence rate established in Theorem \ref{thm:weak convergence}. The right panel shows that the RMSE also decays linearly on the log-log plot, with a slope of approximately $-1/2$,  confirming the strong convergence rate of $O(|\mathscr{G}|^{1/2})$ as stated in Theorem \ref{thm:strong conv}. 
The confidence interval for the strong error is not visible due to the very small variance of the estimator.

\begin{figure}[!ht]
    \centering
    \includegraphics[
    width=\textwidth]{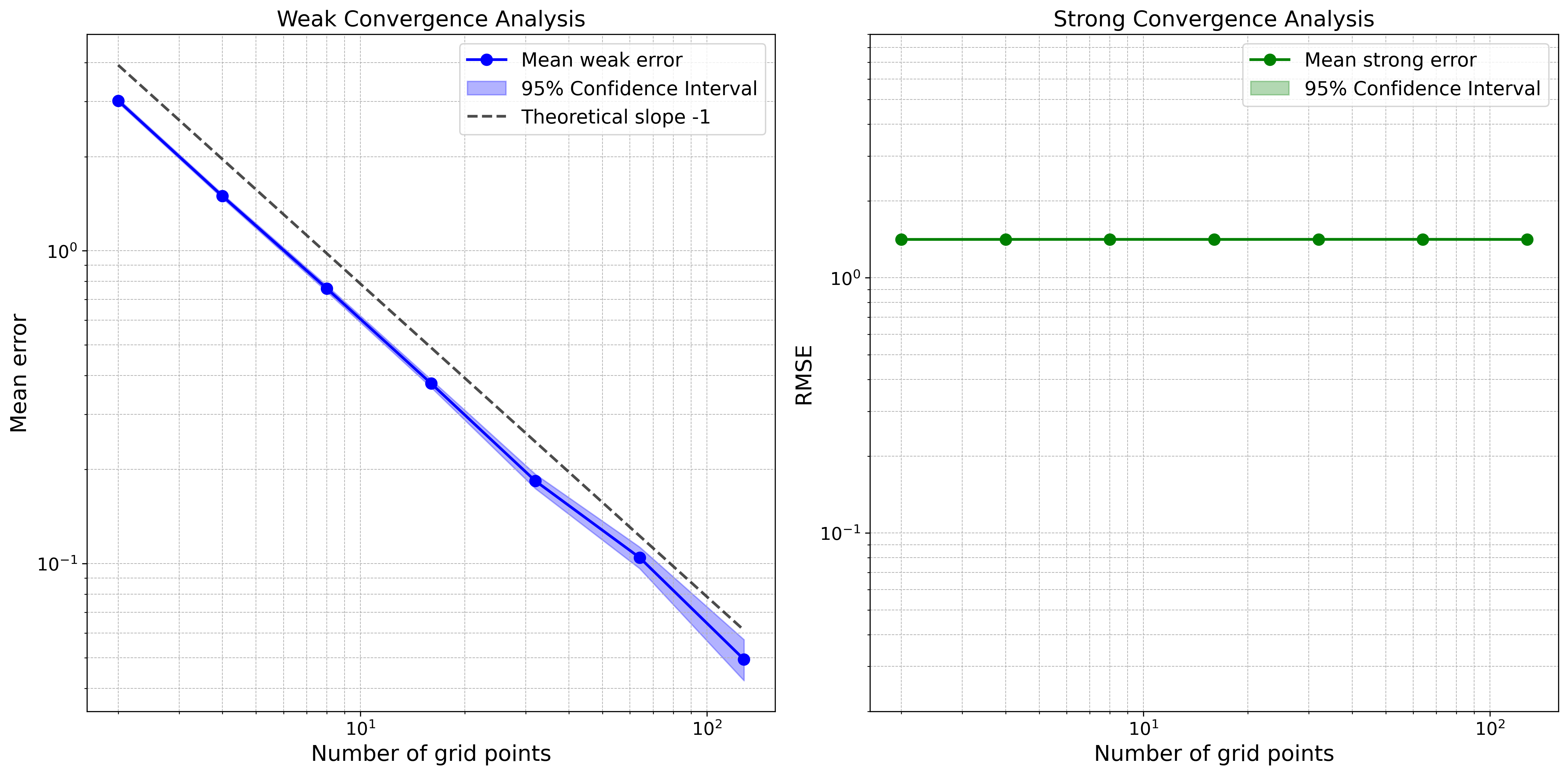}
    \caption{Weak and strong convergence analysis for the controlled volatility case $\dd X_t = a_t \dd W_t$ with actions $a_t\sim \mathcal{N}(0,1)$. The corresponding aggregated dynamics is $\dd \tilde X_t = \dd W_t$. The test function is $f(x)=x^4$ and $T=1$. Left: Weak error versus   the number of grid points $n$. Right: RMSE versus the  number of grid points $n$. Both axes are on log scales.}
    \label{fig:volatility_control}
\end{figure}

Figure \ref{fig:volatility_control} presents the convergence results for the second case, where the volatility is controlled. Similar to the first case with uncontrolled volatility, the left panel shows that the weak error exhibits a first-order decay with a slope of $-1$, consistent with Theorem \ref{thm:weak convergence}. In contrast, the right panel, which depicts strong convergence, shows that the RMSE remains essentially constant as the number of grid points $n$ increases. This provides a clear numerical illustration of the failure of strong convergence under controlled volatility, confirming the theoretical argument presented in Example \ref{eg:counter example}.

}

\section{Applications}
\label{sec:applications}
This section builds on the convergence results from Section \ref{sec:convergence sampled state} to quantify the accuracy of various estimators used in continuous-time RL algorithms. 
The proofs of all statements are given in Section \ref{sec proof: proposition}. 

\subsection{Accuracy of the Monte-Carlo evaluation}

Recall that in RL, the agent has the objective function \eqref{eq:classical objective}, and she needs to estimate its value under a given policy $\bpi$ with data that is generated by executing this policy. This problem is known as the (on-policy) policy evaluation. 

The agent can only discretely sample the actions at the pre-specified time grid and observe the state variable at the same time grid $\mathscr G =\{0=t_0<\ldots <t_n =T\}$. Therefore, the agent can sample the process $\X_s$. For simplicity, we assume that the agent can collect and observe the accumulated running reward $R_{t_{i}}^{\mathscr{G}} = \int_{t_{i}}^{t_{i+1}} r(s,\X_s,\act_{t_i}) \dd s $.\footnote{If only the instantaneous reward $r(s,\X_{t_{i-1}},\act_{t_i})$ can be observed at the grid point, then the accumulated reward needs to be further approximated as $R_{t_{i}}^{\mathscr{G}}\approx r(s,\X_{t_{i-1}},\act_{t_i}) (t_i - t_{i-1})$ as in the usual Euler scheme. Under suitable assumptions on the regularity conditions on the function $r$, the order of our analysis does not change. (cf. \cite{jia2022a})} Hence, the expected reward collected by the agent is 
\begin{equation}\label{eq:value function sampled sde}
    J^{\mathscr{G}}\coloneqq \E \left[ \int_0^T r(s,\X_s,\act_{\delta(s)}) \dd s + h(\X_T)\right] = \E\left[\sum_{i=0}^{n-1} R_{t_{i}}^{\mathscr{G}} + h(\X_T) \right].
\end{equation}
In contrast, the  expected reward with a continuous-time execution of  the policy is given by \eqref{eq:objective entropy relaxed control}:
\begin{equation}\label{eq:value function aggregated sde}
\tilde J\coloneqq
\E \left[ \int_0^T \tilde r(s,\aX_s)\dd s  + h(\aX_T)\right] =
  \E \left[ \int_0^T \int_{A} r(s,\aX_s,a)\bpi(\dd a|s,\aX_s) \dd s  + h(\aX_T)\right].
\end{equation}

The following corollary characterizes the difference between \eqref{eq:value function aggregated sde} and \eqref{eq:value function sampled sde}, that is, the \textit{bias} of using $\int_0^T r(s,\X_s,\act_{\delta(s)}) \dd s + h(\X_T)$ as an estimate for $\tilde J$. 
\begin{Proposition}\label{prop:value function}
Assume (H.\ref{assum:standing}) and (H.\ref{assum:Lipschitz}) hold. 
\begin{enumerate}
    \item \label{enum:value function 1}If  $ r \in C^{2,0}_p, \tilde r \in C^{2}_p$ and (H.\ref{assump:solution_kol_pde}) holds with the terminal function  $f =  h, \tilde r(t,\cdot), 0\le t \le T$,  then there is a constant $ C_p$ depending only on $b,\sigma,\tilde b^{\bpi},\tilde \sigma^{\bpi},\bpi,T,p$  such that
    \begin{equation}
        \left| J^{\mathscr{G}} - \tilde J\right| \le C_p\left(\|h\|_{C_p^4} + \|r\|_{C_p^{2,0}}+\sup_{t\in[0,T]}\|\tilde r(t,\cdot)\|_{C_p^4} + \|\tilde r\|_{C_p^2}\right)|\mathscr{G}|.
    \end{equation} 
    \item \label{enum:value function 2}If (H.\ref{assump:solution Kol pde relax regularity}) holds, and  $r \in \mathcal{H}_T^{(l)}, \tilde r \in \mathcal{H}_T^{(l)}, h\in \mathcal{H}^{(l+2)}$ and $ \tilde r(t,\cdot) \in \mathcal{H}^{(l+2)},\ 0\le t\le T$, then there is a constant $C$ depending only on $b,\sigma,\tilde b^{\bpi},\tilde \sigma^{\bpi},T$ such that
    \begin{equation}
        \left| J^{\mathscr{G}} - \tilde J\right| \le C\left(\|h\|^{(l+2)} + \|r\|_T^{(l)}+\sup_{t\in[0,T]}\|\tilde r(t,\cdot)\|^{(l+2)} + \|\tilde r\|_T^{(l)}\right)|\mathscr{G}|^{\mathcal{X}(l)},
    \end{equation}
    with $ \mathcal{X}(l) $ defined by \eqref{eq:convergence rate}.
\end{enumerate}  
\end{Proposition}

Proposition \ref{prop:value function} characterizes the (unconditional) bias of the estimated value of a given policy by discrete sampling and piecewise constant execution. Given this result, one may understand the required simulation budget in order to obtain an estimate with desired accuracy. Next, we use the notion of the conditional weak error to highlight that the required simulation budget can be refined to those conditioned on the policy sampler. More precisely, consider the value conditional on $\xi$,
\begin{equation}\label{eq:value function sampled sde conditional}
    J^{\mathscr{G}}(\xi)\coloneqq \Ew \left[ \int_0^T  r(s,\X_s,\act_{\delta(s)}) \dd s + h(\X_T)\right].
\end{equation}
Combining the proof of Proposition \ref{prop:value function} and Theorem \ref{thm:concentration inequality}, we obtain the following result for the large deviation bound for the value function.
\begin{Proposition}
   Assume (H.\ref{assum:standing}) and (H.\ref{assum:Lipschitz}) hold. If  (H.\ref{assump:linear growth on x}) and the conditions in  Proposition \ref{prop:value function}(\ref{enum:value function 1})  hold, and $ \sup_{a\in A}\|r(\cdot,\cdot,a)\|_{C_p^{2}} < +\infty$, then there is a constant $ C_p$ depending only on $b,\sigma,\tilde b^{\bpi},\tilde \sigma^{\bpi},T,p$ such that for every $ \varrho \in (0,1)$,
    \begin{equation}
        \PP\left[ |J^{\mathscr{G}}(\xi) - \tilde J| > C_p\left(\|h\|_{C_p^4} + \sup_{a\in A}\|r(\cdot,\cdot,a)\|_{C_p^{2}}+\sup_{t\in[0,T]}\|\tilde r(t,\cdot)\|_{C_p^4} + \|\tilde r\|_{C_p^2}\right)\sqrt{|\mathscr{G}|}(1+\sqrt{\ln(2/\varrho)})\right] \le \varrho.
    \end{equation}
   If conditions in  Proposition \ref{prop:value function}(\ref{enum:value function 2})   holds, then there is a constant $C$ depending only on $b,\sigma,\tilde b^{\bpi},\tilde \sigma^{\bpi},T$ such that for every $ \varrho \in (0,1)$,
    \begin{equation}
        \PP\left[ |J^{\mathscr{G}}(\xi) - \tilde J| > C\left(\|h\|^{(l+2)} + \|r\|_T^{(l)}+\sup_{t\in[0,T]}\|\tilde r(t,\cdot)\|^{(l+2)} + \|\tilde r\|_T^{(l)}\right)|\mathscr{G}|^{\mathcal{X}(l)\land \frac{1}{2}}(1+\sqrt{\ln(2/\varrho)})\right] \le \varrho 
    \end{equation}
     with $ \mathcal{X}(l) $ defined by \eqref{eq:convergence rate}.
\end{Proposition}

\subsection{Accuracy of Temporal-Difference learning} The \textit{temporal difference} (TD) plays a crucial role in many RL algorithms as the policy evaluation methods. In contrast to the Monte-Carlo evaluation, TD learning often aims to find an approximation of the value function associated with a given policy $\bpi$, defined as
\begin{equation}
\label{eq:value func relaxed}
\tilde J(t,x;\bpi) = \E \left[ \int_t^T \tilde r(s,\aX_s)\dd s  + h(\aX_T) \Big| \aX_t = x \right],
\end{equation}
within a suitable parametric family of functions, and it makes use of the increment of the value function process. 

In particular, suppose $V^{\theta}$ is a family of parameterized functions of $(t,x)$ with $\theta$ a finite-dimensional parameter. TD learning concerns the increment in the form ``$\dd V^{\theta}(t, \aX_t)$'' and the instantaneous reward ``$r(t,\aX_t)\dd t$''. In the continuous-time framework, these two increments have different natures. ``$\dd V^{\theta}(t, \aX_t)$'' is merely an informal notation that shall be combined with a suitable integrand as in the definition of It\^o's integral, and the latter ``$r(t,\aX_t)\dd t$'' is the usual Lebesgue integral with respect to time. In theory, \cite{jia2022b} identify the continuous-time counterpart of TD learning as solving the martingale orthogonality condition. More precisely, if $\theta^*$ is the root to  
\begin{equation}\label{eq:martingale orth equation}
    \Ew\left[ \int_0^T S(t,\aX_t) \left[ \dd V^{\theta}(t,\aX_t) +\tilde r(t,\aX_t) \dd t\right]\right] = 0,
\end{equation}
for all admissible test functions $S(t,x)$, then the $V^{\theta^*}(t,x) = \tilde J(t,x;\bpi)$ is the desired value function associated with a given policy $\bpi$. Therefore, various RL algorithms are based on using stochastic approximation (cf. \cite{borkar2000ode}) to find the root $\theta^*$ to the martingale orthogonality condition \eqref{eq:martingale orth equation}. The most popular choice is the so-called ``TD(0)'' algorithm, which corresponds to taking $S(t,x) = \frac{\partial V^{\theta}}{\partial \theta}(t,x)$.

For a given test function $S$ and parameter $\theta$, applying the stochastic approximation algorithm means approximating this expectation with random samples, i.e., an estimator of this expectation. The typical convergence conditions of the stochastic approximation algorithm are threefold\footnote{Besides these conditions on the accuracy, the convergence also depends on the landscape of \eqref{eq:martingale orth equation} as a function of $\theta$. However, this is jointly determined by the parametrization  $\theta\mapsto V^{\theta}$ and the properties of the aggregated dynamics, which is not affected by time discretization.}: first, the estimator has diminishing bias; second, the estimator has non-exploding variances; third, the learning rates should be suitably chosen according to the bias and variance of the estimator. Therefore, it is crucial to understand the magnitude of bias and variance and how they depend on the grid size. 

In continuous time, however, due to the discretely sampled actions and the nature of the stochastic integral inside \eqref{eq:martingale orth equation}, the agent only has access to the sampled dynamics and uses the discrete sum (we drop the superscript $\theta$ in the following to illustrate the results for a generic function $V$)
\[ S \circ V := \sum_{i=0}^{n-1}S(t_{i},\X_{t_{i}})\left( V(t_{i+1},\X_{t_{i+1}}) - V(t_i,\X_{t_i}) + R_{t_i}^{\mathscr{G}}\right), \]
with the observed accumulated reward 
$R_{t_{i}}^{\mathscr{G}} = \int_{t_{i}}^{t_{i+1}} r(s,\X_s,\act_{t_i}) \dd s $ 
as an approximation for the expectation in \eqref{eq:martingale orth equation}. We characterize the bias and variance of such an approximation in Proposition \ref{prop: martingale orth no a application}.
 
\begin{Proposition}\label{prop: martingale orth no a application}
Suppose (H.\ref{assum:standing}) and (H.\ref{assum:Lipschitz}) hold. Assume that $ V \in C_p^2$. Denote 
\begin{equation}\label{eq:g}
     g(t,x,a) = \partial_t V(t,x) + b^{\top}(t,x,a)\partial_x V(t,x) + \frac{1}{2}\tr\left(\sigma\sigma^{\top}(t,x,a)  \textrm{Hess}_xV(t,x)\right) + r(t,x,a)=:\mathcal{L}^a V(t,x).
\end{equation}
For the variance of $ S\circ V$, if $ S \in C_p^{0}$ and $r\in C^{0,0}_p$, then
    \begin{equation}
        \var{S\circ V} \le C_p\|S\|_{C_p^{0}}^2(\|V\|_{C_p^2}+\|r\|_{C_p^{0,0}})^2,
    \end{equation}
    where $ C_p$ is a constant depending only on $b,\sigma,\tilde b^{\bpi},\tilde \sigma^{\bpi},\bpi,T,p$.\\

For the bias of $ S\circ V$, we have the following holds.
    \begin{enumerate}
    \item[(1)] If $ g \in C_p^{0,0}, S \in C^{2}_p$ and $gS$ satisfies the condition (\ref{enum:value function 1}) in Proposition \ref{prop:value function} as the running term, then there is a constant $C_p$ depending only on $b,\sigma,\tilde b^{\bpi}, \tilde \sigma^{\bpi},\bpi,T,p$ such that
    \begin{equation}
    \begin{aligned}
        &\left| \E[S \circ V] - \Ew\left[ \int_0^T S(t,\aX_t)\left[\dd V(t,\aX_t) + \tilde r(t,\aX_t)\dd t\right]\right] \right| \\
        &\le C_p\left(\|S\|_{C_p^{2}}\|g\|_{C_p^{0,0}}+\|Sg\|_{C_p^{2,0}} + \|\widetilde{Sg}\|_{C_p^2} + \sup_{t\in[0,T]}\|\widetilde{Sg}(t,\cdot)\|_{C_p^4}\right)|\mathscr{G}|,
        \end{aligned}
    \end{equation}
    where $\widetilde{Sg}(t,x) := \int_A Sg(t,x,a)\bpi(\dd a|t,x)$.
    \item[(2)] 
    If $ g\in \mathcal{H}_T^{(0)}, S \in \mathcal{H}_T^{(l)}$ and $ gS$ satisfies the condition (\ref{enum:value function 2}) in Proposition \ref{prop:value function} as the running term, then there is a constant $C$  depending only on $b,\sigma,\tilde b^{\bpi},\tilde \sigma^{\bpi},T$ such that
    \begin{equation}
    \begin{aligned}
        &\left|\E[S \circ V ] - \Ew\left[ \int_0^T S(t,\aX_t)\left[\dd V(t,\aX_t) + \tilde r(t,\aX_t)\dd t\right]\right] \right| \\
        &\le C\left(\|S\|_T^{(l)}\|g\|_T^{(0)}+\|Sg\|^{(l)}_T + \|\widetilde{Sg}\|^{(l)}_T + \sup_{t\in[0,T]}\|\widetilde{Sg}(t,\cdot)\|^{(l+2)}\right)|\mathscr{G}|^{\mathcal{X}(l)}.
    \end{aligned}
    \end{equation} 
    \end{enumerate}
\end{Proposition}

Proposition \ref{prop: martingale orth no a application} implies that we can obtain a more accurate stochastic approximation with a refined grid. In addition, the variance of such approximate is uniformly bounded in the grid size. It means that with a suitable schedule to refine the grid, we may expect the typical convergence conditions in the stochastic approximation algorithm will be satisfied. 

\subsection{Accuracy of policy gradient}
The policy gradient is one of the most commonly used methods in policy optimization in RL. The well-celebrated (model-free) policy gradient representation (e.g., \cite{sutton1999policy}) gives an unbiased estimator for the gradient of the value function with respect to certain parameters in the policy. For example, the stochastic policy $\bpi^{\psi}$ is parameterized by $\psi$, a finite-dimensional parameter, and it admits a density function with respect to the Lebesgue measure. We denote the density function by $H^{\psi}(t,x,a) = \frac{\bpi^{\psi}(\dd a|t,x)}{\mu(\dd a)}$. Then we know that the value function under the policy $\bpi^{\psi}$ (defined in \eqref{eq:value func relaxed}) will depend on $\psi$ implicitly, denoted by $\tilde J(t,x;\bpi^{\psi})$. The policy gradient concerns about the derivative of $\tilde J(t,x;\bpi^{\psi})$ with respect to $\psi$, i.e., $\frac{\partial \tilde J(t,x;\bpi^{\psi})}{\partial \psi}$. For simplicity, we focus on the initial time-state gradient and suppress the time-state pair's dependence. 

A heuristic derivation in \cite{jia2022b} shows that 
the gradient  $\frac{\partial \tilde J(\bpi^{\psi})}{\partial \psi}$  
can be represented by
\begin{align}
&\E\left[ \int_0^T  \frac{\partial \log H^{\psi}}{\partial \psi}(a_t|t,X_t) \left[ \dd \tilde J(t,X_t;\bpi^{\psi}) + r(t,X_t,a_t)\dd t \right]   \right] \notag\\
&= \Ew\left[ \int_0^T \int_A\frac{\partial \log H^{\psi}}{\partial \psi}(a|t,\aX_t)\mathcal{L}^a \tilde J(t,\aX_t;\bpi^{\psi})\bpi(\dd a|t,\aX_t)\dd t\right] .
\label{eq:policy gradient repre}
\end{align}
The expectation \eqref{eq:policy gradient repre} can be approximated by the forward-Euler-type of finite sums using sampled dynamics, i.e.,
\[\sum_{i=0}^{n-1} \frac{\partial \log H^{\psi}}{\partial \psi}(a_{t_i}|t_i,\X_{t_i}) \left[  \tilde J(t_{i+1},\X_{t_{i+1}};\bpi^{\psi}) - \tilde J(t_i,\X_{t_i};\bpi^{\psi})  +  R_{t_i}^{\mathscr{G}} \right] .   \]

Naturally, the convergence of gradient-based algorithms depends on whether we have an accurate gradient estimate. To  analyze the accuracy of the policy gradient estimator, we   consider the following general expression:
\begin{equation}
\label{eq:pg discrete sum}
S \circ V := \sum_{i=0}^{n-1}S(t_{i},\X_{t_{i}},a_{t_i})\left( V(t_{i+1},\X_{t_{i+1}}) - V(t_i,\X_{t_i}) + R_{t_i}^{\mathscr{G}}\right),
\end{equation}
where 
$R_{t_{i}}^{\mathscr{G}} = \int_{t_{i}}^{t_{i+1}} r(s,\X_s,\act_{t_i}) \dd s $ is the observed accumulated reward over $[t_i,t_{i+1}]$. 

The form \eqref{eq:pg discrete sum} also appears in the continuous-time q-learning algorithm proposed in \cite{jia_q-learning_2022}, where $S$ is the derivative of the q-function that needs to be learned. The difference between \eqref{eq:pg discrete sum} and \eqref{eq:martingale orth equation} is that now the function $S$ also depends on the sampled action $a_{t_i}$. Thus, \eqref{eq:pg discrete sum} is not a straightforward substitution and discretization of any integration along the aggregated dynamics and can only be well-defined using the sampled dynamics. Note that given $\X_{t_i}$, the value of $\X_{t_{i+1}}$ not only depends on the environmental noise (the underlying Brownian motions) but also depends on $a_{t_i}$. The subtle difference in this measurability condition will affect the limiting value of \eqref{eq:pg discrete sum}.
 
We characterize the limit of a general estimator in the form \eqref{eq:pg discrete sum} and quantify the magnitude of its bias and variance in the next proposition. The proof is similar to that of Proposition \ref{prop: martingale orth no a application}, so we omit it.

\begin{Proposition}
\label{prop:pg}
     Suppose (H.\ref{assum:standing}) and (H.\ref{assum:Lipschitz}) hold. Assume $V\in C_p^2$. Define $g$ by \eqref{eq:g}. For the variance of $ S\circ V$, if $ S \in C_p^{0,0}$ and $ r\in C_p^{0,0}$, then
    \begin{equation}
        \var{S\circ V} \le C_p\|S\|_{C_p^{0,0}}^2(\|V\|_{C_p^2}+\|r\|_{C_p^{0,0}})^2,
    \end{equation}
    where $ C_p$ is a constant depending only on $b,\sigma,\tilde b^{\bpi}, \tilde \sigma^{\bpi},\bpi,T, p$.

For the bias of $ S\circ V$, we have the following holds.
\begin{enumerate}[(1)]
\item If $ g\in C_p^{0,0}, S \in C^{2,0}_p$ and $gS$ satisfies the condition (\ref{enum:value function 1}) in Proposition \ref{prop:value function} as the running term, then there is a constant $C_p$ depending only on $b,\sigma,\tilde b^{\bpi},\tilde \sigma^{\bpi},\bpi,T, p$ such that
    \begin{equation}
    \begin{aligned}
        &\left| \E[S \circ V] - \Ew\left[ \int_0^T \widetilde{Sg}(t,\aX_t)\dd t\right] \right|\\
        &\le C_p\left(\|S\|_{C_p^{2,0}}\|g\|_{C_p^{0,0}}+\|Sg\|_{C_p^{2,0}} + \|\widetilde{Sg}\|_{C_p^2} + \sup_{t\in[0,T]}\|\widetilde{Sg}(t,\cdot)\|_{C_p^4}\right)|\mathscr{G}|,
        \end{aligned}
    \end{equation}
    where $ \widetilde{Sg}(t,x):= \int_\A S(t,x,a)g(t,x,a)\bpi(\dd a|t,x)$. 

 In particular, if the conditions in \cite{jia2022b} are satisfied so that the policy gradient has the representation \eqref{eq:policy gradient repre}, then by setting $V(\cdot,\cdot) = \tilde J(\cdot,\cdot;\bpi^{\psi})$, and $S(t,x,a) = \frac{\partial \log H^{\psi}}{\partial \psi}(a|t,x)$, we have
  \[ \left| \E\left[ S\circ V \right] - \frac{\partial \tilde J(\bpi^{\psi})}{\partial \psi} \right| \leq C_p\left(\|S\|_{C_p^{2,0}}\|g\|_{C_p^{0,0}}+\|Sg\|_{C_p^{2,0}} + \|\widetilde{Sg}\|_{C_p^2} + \sup_{t\in[0,T]}\|\widetilde{Sg}(t,\cdot)\|_{C_p^4}\right) |\mathscr{G}| . \]
\item  
If $ g \in \mathcal{H}_T^{(0)}, S \in \mathcal{H}_T^{(l)}$ and $ gS$ satisfies the condition (\ref{enum:value function 2}) in Proposition \ref{prop:value function} as the running term, then there is a constant $C$ depending only on $b,\sigma,\tilde b^{\bpi},\tilde \sigma^{\bpi},T$ such that
    \begin{equation}
    \begin{aligned}
        &\left| \E[S\circ V] - \Ew\left[ \int_0^T \widetilde{Sg}(t,\aX_t)\dd t\right]\right| \\
        &\le C\left(\|S\|_T^{(l)}\|g\|_T^{(0)}+\|Sg\|^{(l)}_T + \|\widetilde{Sg}\|^{(l)}_T + \sup_{t\in[0,T]}\|\widetilde{Sg}(t,\cdot)\|^{(l+2)}\right)|\mathscr{G}|^{\mathcal{X}(l)}.
        \end{aligned}
    \end{equation}
    
In particular, if the conditions in \cite{jia2022b} are satisfied so that the policy gradient has the representation \eqref{eq:policy gradient repre}, then by setting $V(\cdot,\cdot) = \tilde J(\cdot,\cdot;\bpi^{\psi})$, and $S(t,x,a) = \frac{\partial \log H^{\psi}}{\partial \psi}(a|t,x)$, we have
  \[ \left| \E\left[ S\circ V \right] - \frac{\partial \tilde J(\bpi^{\psi})}{\partial \psi} \right| \le C\left(\|S\|_T^{(l)}\|g\|_T^{(0)}+\|Sg\|^{(l)}_T + \|\widetilde{Sg}\|^{(l)}_T + \sup_{t\in[0,T]}\|\widetilde{Sg}(t,\cdot)\|^{(l+2)}\right)|\mathscr{G}|^{\mathcal{X}(l)}.  \]
\end{enumerate}
 
\end{Proposition}

Proposition \ref{prop:pg} implies that with a refined grid, we can obtain a more accurate estimate for the policy gradient. In addition, the variance of such an estimator is uniformly bounded in the grid size. 
These properties are crucial for adjusting the sampling frequency to ensure the convergence of the overall policy gradient algorithm.\footnote{The typical conditions of convergence of the stochastic gradient descent algorithm also require $\frac{\partial \tilde J(\bpi^{\psi})}{\partial \psi}$ to satisfy certain conditions as a function of $\psi$, but it is not affected by time discretization. Note that in a typical actor-critic algorithm, one does not have an exact estimate of $\tilde J(\cdot,\cdot;\bpi^{\psi})$, and can only obtain its approximation $V^{\theta} \approx \tilde J(\cdot,\cdot;\bpi^{\psi})$ via policy evaluation. Therefore, the convergence of a general actor-critic algorithm is more complex and open for future research.}

\subsection{Accuracy of risk-sensitive q-learning}
So far, all the discussions have been restricted to the additive functional, which is often regarded as a risk-neutral problem. In continuous-time risk-sensitive problems (e.g., \cite{bielecki1999risk}), one encounters the objective function in the exponential form:
\[ \frac{1}{\epsilon}\log\E\left[ \exp\left\{ \epsilon \left[ \int_0^T r(s,X_s,a_s)\dd s + h(X_T) \right] \right\} \right] . \]
\cite{jia2024continuous} establishes the risk-sensitive q-learning theory and reveals that it only differs from the non-risk-sensitive counterpart by a penalty on the quadratic variation of the value function. Hence, the RL algorithm for the non-risk-sensitive problems can be easily modified to account for risk sensitivity by adding an integral with respect to a suitable quadratic variation process, which can be approximated by the sampled dynamics as
\[  S \Box V :=\sum_{i=0}^{n-1} S(t_{i},\X_{t_{i}},a_{t_i})\left( V(t_{i+1},\X_{t_{i+1}}) - V(t_{i},\X_{t_{i}}) + R_{t_i}^{\mathscr{G}} \right)^2. \]

\begin{Proposition}\label{prop:quadratic variation application}
    Suppose (H.\ref{assum:standing}) and (H.\ref{assum:Lipschitz}) hold. Assume $ V \in C_p^2$. Define $g_1$ by \eqref{eq:g} and $g_2(t,x,a)= \left| \sigma^{\top}\partial_xV(t,x,a)\right|^2$. For the variance of $ S\circ V$, if $ S, g_1\in C_p^{0,0}$, then
    \begin{equation}
        \var{S\Box V} \le  C_p \|S\|_{C_p^{0,0}}^2\left( \|g_1\|_{C_p^{0,0}}^2 + \|V\|_{C_p^2}^2 \right)^2,
    \end{equation}
    where $ C_p$ is a constant depending only on $b,\sigma,\tilde b^{\bpi}, \tilde \sigma^{\bpi},\bpi,T,p$.
    
For the bias of $ S\Box V$, we have the following holds.
\begin{enumerate}[(1)]
\item If $S, g_1, \partial_x V, \sigma \in C^{2,0}_p, g_2\in C_p^{0,0}$ and $ Sg_2$ satisfies the condition (\ref{enum:value function 1}) on Proposition \ref{prop:value function} as the running term, then there is a constant $C_p$ depending only on $b,\sigma,\tilde b^{\bpi},\tilde \sigma^{\bpi},\bpi,T,p$ such that
    \begin{equation}
    \begin{aligned}
        &\left| \E[S \Box V] - \E\left[ \int_0^T \widetilde{Sg_2}(t,\aX_t)\dd t \right]\right| \\
        &\le C_p\left(\|S\|_{C_p^{2,0}}(\|g_1\|_{C_p^{2,0}}^2+\|g_2\|_{C_p^{0,0}})+\|g_1\|_{C_p^{2,0}}\|V\|_{C_p^2}+ \|Sg_2\|_{C_p^{2,0}}+ \|\widetilde{Sg_2}\|_{C_p^2} + \sup_{t\in[0,T]}\|\widetilde{Sg_2}(t,\cdot)\|_{C_p^4}\right)|\mathscr{G}|.
        \end{aligned}
    \end{equation}
\item 
If $S, g_1, \partial_x V, \sigma \in \mathcal{H}^{(l)}_T, g_2\in \mathcal{H}_T^{(0)}$ and  $ Sg_2$ satisfies the condition (\ref{enum:value function 2}) on Proposition \ref{prop:value function} as the running term, then there is a constant $C$ depending only on $b,\sigma,\tilde b^{\bpi},\tilde \sigma^{\bpi},T$ such that
    \begin{equation}
    \begin{aligned}
         &\left| \E[S \Box V] - \E\left[ \int_0^T \widetilde{Sg_2}(t,\aX_t) \dd t \right]\right| \\
         &\le C\left(\|S\|_T^{(l)}((\|g_1\|_T^{(l)})^2+\|g_2\|_T^{(0)})+\|g_1\|_T^{(l)}\|V\|_T^{(l)}+ \|Sg_2\|_T^{(l)}+ \|\widetilde{Sg_2}\|_T^{(l)} + \sup_{t\in[0,T]}\|\widetilde{Sg_2}(t,\cdot)\|^{(l+2)}\right)|\mathscr{G}|^{\mathcal{X}(l)}.
    \end{aligned}
    \end{equation}
\end{enumerate}
In particular, if $ S$ does not depend on $ a$, then
\begin{equation}\notag
    \E\left[ \int_0^T \int_A \widetilde{Sg_2}(t,\aX_t)\dd t \right] = \E\left[ \int_0^T S(t,\aX_t)\dd \left<V(\cdot,\aX_\cdot)\right>_t\right],
    \end{equation}
where $t\mapsto \left<V(\cdot,\aX_\cdot)\right>_t$ is the quadratic variation of the process $t\mapsto V(t,\aX_t)$.
\end{Proposition}

\section{Proofs of Main Results}\label{sec:proofs}

\subsection{Proofs of  Theorems \ref{thm:weak convergence}, 
\ref{thm:weak convergence relax regularity},
\ref{thm:concentration inequality}, and
 \ref{thm:strong conv}
}
\subsubsection{Proof of Theorem \ref{thm:weak convergence}}\label{sec proof: thm weak convergence}

\change{
The following lemma can be obtained by using It\^o's formula and Lemma \ref{lemma:moments of X}. The proof is given in 
Appendix 
\ref{sec proof: lemma}.


\begin{Lemma}\label{lemma:convergence rate of diff}
    Suppose (H.\ref{assum:standing}) holds, and $g \in C_p^{2,0}$. Then there is a constant $C\ge 0$  such that for all grids $\mathscr G$ and  $s \in [0,T]$, 
    \begin{equation}
        \left| \E\left[ g(s,\X_s,a_{\delta(s)}) - g(\delta(s),\X_{\delta(s)},a_{\delta(s)})\right]\right| \le C\|g\|_{C_p^{2,0}}|\mathscr{G}|.
    \end{equation}
\end{Lemma}
Based on Lemma \ref{lemma:convergence rate of diff}, we prove Theorem \ref{thm:weak convergence}.
}
\begin{proof}[Proof of Theorem \ref{thm:weak convergence}]
Without loss of generality, assume $ t^{\prime} = t_{j}$, where  $1\le j\le n$. Consider the following  PDE 
\begin{equation}\label{eq:pde_feymann_j}
            \frac{\partial}{\partial t} v(t,x) + \mathcal{L} v(t,x) = 0,\quad t\in [0,t_j]\times \sR^d; \quad
            v(t_j,x) = f(x),\quad x\in \sR^d,
    \end{equation}
\change{where $ \mathcal{L}$ is  the generator of the aggregated dynamic $ \aX$ \eqref{eq:aggregated_sde} and is defined by \eqref{eq: generator L}.} Denote by $v: [0,t_f]\to \sR$   the solution of~\eqref{eq:pde_feymann_j}.
  By (H.\ref{assump:solution_kol_pde}), $v $ is smooth and $\|v \|_{C_p^4}<+\infty$.

By the Feynman-Kac formula, $\E[f(\aX_{t_j})] = v(0,x_0)$, and hence  \begin{equation}\label{eq:sum of diff X weak convergence}
    \begin{aligned}
        \E\left[ f(\X_{t_j}) - f(\aX_{t_j})\right] &= \E\left[ v(t_j,\X_{t_j}) - v(0,x_0)\right]\\\
        & = \E\left[ \sum_{i = 1}^j\left(v(t_i,\X_{t_i}) - v(t_{i-1},\X_{t_{i-1}})\right)\right]  =: \sum_{i = 1}^j\E\left[ e_i\right].
    \end{aligned}
    \end{equation}
    It suffices to estimate each $\E[e_i]$. For simplicity, denote $\partial_t:= \partial/\partial t$ and $\partial_x:=\partial/\partial x$. It follows from It\^o's lemma (see the dynamic~\eqref{eq:sample_sde_abbre} of $\X$) that 
    \begin{align}
        \E\left[ e_i\right] =& \E\left[ v(t_i,\X_{t_i}) - v(t_{i-1},\X_{t_{i-1}})\right] \notag\\ 
         =& \E\bigg[ \inti \left(\partial_t v(s,\X_s)  + b^{\top}(s,\X_{s},\act_{t_{i-1}})\partial_x v(s,\X_s) \right)\dd s 
        \notag\\
        & \quad + \inti \frac{1}{2} \tr\left(\sigma\sigma^{\top}(s,\X_{s},\act_{t_{i-1}})\partial_x^2v(s,\X_s)\right) \dd s\bigg],\notag
    \end{align}
    with $\act_{t_{i-1}} \sim \bpi(\cdot | t_{i-1},\X_{t_{i-1}})$, where   the second equality used the fact that the It\^o integral is a martingale, given that $\sigma$ and $\partial_x v$ have polynomial growth and the moments of $\X_s$ are finite (by Lemma~\ref{lemma:moments of X}). Noting that $\xi_{i-1}$ is independent of $\X_{t_{i-1}}$, we have
    \begin{equation}\label{mideq:by pde equals to 0}
    \begin{aligned}
        &\E\bigg[ \inti \left(\partial_t v(t_{i-1},\X_{t_{i-1}})  + b^{\top}(t_{i-1},\X_{t_{i-1}},\act_{t_{i-1}})\partial_x v(t_{i-1},\X_{t_{i-1}}) \right)\dd s 
        \\
        & \quad + \inti \frac{1}{2} \tr\left(\sigma\sigma^{\top}(t_{i-1},\X_{t_{i-1}},\act_{t_{i-1}})\partial_x^2v(t_{i-1},\X_{t_{i-1}})\right) \dd s\bigg] \\
        & =\E\bigg[ \inti \left(\partial_t v(t_{i-1},\X_{t_{i-1}})  + (\tilde b^{\bpi})^{\top}(t_{i-1},\X_{t_{i-1}})\partial_x v(t_{i-1},\X_{t_{i-1}}) \right)\dd s 
        \\
        & \quad + \inti \frac{1}{2} \tr\left((\tilde\sigma^{\bpi})^2(t_{i-1},\X_{t_{i-1}})\partial_x^2v(t_{i-1},\X_{t_{i-1}})\right) \dd s\bigg]\change{= 0}, 
    \end{aligned}
    \end{equation}
    where the first equality follows from taking conditional expectation $\E[ \cdot| \X_{t_{i-1}}]$ and the second equality follows from~\eqref{eq:pde_feymann_j}. Therefore,
    \change{
    \begin{equation}\label{mideq:decomp e_i}
    \begin{aligned}
        \E\left[ e_i\right] &= \E\bigg[ \inti\left(\partial_tv(s,\X_s) - \partial_t v(t_{i-1},\X_{t_{i-1}})\right) \dd s\bigg]\\
        +&\E\bigg[\inti \left(b^{\top}(s,\X_{s},\act_{t_{i-1}})\partial_x v(s,\X_s) - b^{\top}(t_{i-1},\X_{t_{i-1}},\act_{t_{i-1}})\partial_x v(t_{i-1},\X_{t_{i-1}}) \right)\dd s\bigg]\\
        +& \E\bigg[\inti \frac{1}{2}\tr\left( \sigma\sigma^{\top}(s,\X_s,\act_{t_{i-1}})\partial_x^2 v(s,\X_{s}) - \sigma\sigma^{\top}(t_{i-1},\X_{t_{i-1}},\act_{t_{i-1}})\partial_x^2 v(t_{i-1},\X_{t_{i-1}})\right) \dd s\bigg].
    \end{aligned}
    \end{equation}
    Based on (H.\ref{assump:solution_kol_pde}), $\partial_t v, b^{\top}\partial_x v, \frac{1}{2}\tr(\sigma\sigma^{\top}\partial_x^2 v) \in C_p^{2,0}$. It follows from Lemma \ref{lemma:convergence rate of diff} that 
    \begin{equation}\label{eq: e bound}
        \begin{aligned}
            \left| \E\left[ e_i\right]\right| &\le \inti C\left( \|\partial_t v\|_{C_p^2} + \|b^{\top}\partial_xv\|_{C_p^{2,0}}+\frac{1}{2}\|\tr( \sigma\sigma^{\top}\partial_x^2 v)\|_{C_p^{2,0}}\right)|\mathscr{G}| \dd s\\
            &\le   C_p\|v\|_{C_p^4}|\mathscr{G}|(t_i - t_{i-1})
        \end{aligned}
    \end{equation}
    for a constant  $C_p$   depending only on $b,\sigma,p$.  
    The desired result follows from \eqref{eq:sum of diff X weak convergence},
    \eqref{eq: e bound}, 
    and $\|v\|_{C_p^4}\le C\|f\|_{C_p^4}$ due to (H.\ref{assump:solution_kol_pde}).}
\end{proof}

\subsubsection{Proof of Theorem \ref{thm:weak convergence relax regularity}}\label{sec proof: thm weak convergence relax regularity}

The following lemma is similar to \cite[Lemma 14.1.6]{kloeden1992stochastic}. One can combine the method in Theorem \ref{thm:weak convergence} to prove it, so we omit the proof here.
\begin{Lemma}\label{lemma:convergence rate of diff relax regularity}
    Suppose   (H.\ref{assum:standing}) holds, $l \in (0,1)\cup(1,2)\cup(2,3)$ and $g\in \mathcal{H}_T^{(l)}$. Then there \change{is} a constant $C$, such that for all $s \in [0,T]$,
    \begin{equation}
        \left| \E\left[ g(s,\X_s,\act_{\delta(s)}) - g(\delta(s),\X_{\delta(s)},\act_{\delta(s)})\right]\right|\le C\|g\|_T^{(l)}|\mathscr{G}|^{\mathcal{X}(l)},
    \end{equation}
    where $\mathcal{X}(l)$ satisfies \eqref{eq:convergence rate}.
\end{Lemma}
Based on Lemma \ref{lemma:convergence rate of diff relax regularity}, we prove Theorem \ref{thm:weak convergence relax regularity}.
\begin{proof}[Proof of Theorem \ref{thm:weak convergence relax regularity}]
    It suffices to prove $| \E[ f(\X_{t^\prime})] - \E[ f(\aX_{t^\prime})]| \le C\|f\|^{(l+2)} |\mathscr{G}|^{\mathcal{X}(l)} $ for any $t^{\prime}\in [0,T]$. Under (H. \ref{assump:solution Kol pde relax regularity}), by \cite[Chapter 4, Theorem 5.1]{ladyzhenskaya1968linear}, for any $f\in \mathcal{H}^{(l+2)}$, the PDE \eqref{eq:Kol_pde} has a unique solution $ v(\cdot,\cdot;t') \in \mathcal{H}_{t'}^{(l+2)}$ and $ \|v(\cdot,\cdot;t')\|_{t'}^{(l+2)} \le C\|f\|^{(l+2)}$, where $ C$ is a constant depending only on $T, \tilde b^{\bpi}, \tilde \sigma^{\bpi}$. Note that $ \mathcal{H}_T^{(l)} \subset \mathcal{H}_{\change{t^{\prime}}}^{(l)}$ for every $ 0\le t^{\prime}\le T$. Without loss of generality, suppose that $ t^{\prime} = T$. For simplicity, we drop the notation $t^{\prime}$  and use $v$ to represent the solution of \eqref{eq:Kol_pde} with $ t^{\prime} = T$. Recall \eqref{eq:sum of diff X weak convergence},
    \begin{align}
        \E\left[ f(\X_T) - f(\aX_T)\right] &= \E\left[ v(T,\X_T) - v(0,x_0)\right] = \sum_{i = 1}^n\E\left[ e_i\right].\notag
    \end{align}
    Note that \eqref{mideq:by pde equals to 0} still holds in this situation. Therefore, we have
    \begin{equation}\label{mideq:decomp e_i relax regularity}
    \begin{aligned}
        \E&\left[ e_i\right] = \E\bigg[ \inti\left(\partial_tv(s,\X_s) - \partial_t v(t_{i-1},\X_{t_{i-1}})\right) \dd s\bigg]\\
        & +\E\bigg[\inti b^{\top}(s,\X_{s},\act_{t_{i-1}})\partial_x v(s,\X_s) - b^{\top}(t_{i-1},\X_{t_{i-1}},\act_{t_{i-1}})\partial_x v(t_{i-1},\X_{t_{i-1}}) \dd s\bigg]\\
        & +\E\bigg[\inti \frac{1}{2}\tr\left(\sigma\sigma^{\top}(s,\X_s,\act_{t_{i-1}})\partial_x^2 v(s,\X_s) - \sigma\sigma^{\top}(t_{i-1},\X_{t_{i-1}},\act_{t_{i-1}})\partial_x^2 v(t_{i-1},\X_{t_{i-1}}) \right)\dd s\bigg].
    \end{aligned}
    \end{equation}
    Note that $\partial_t v, b^{\top}\partial_x v$ and $\frac{1}{2}\tr(\sigma\sigma^{\top}\partial_x^2 v)$ belong to $\mathcal{H}_T^{(l)}$. By Lemma \ref{lemma:convergence rate of diff relax regularity}, the three terms on the right hand side of above equation are all bounded by $C\|v\|^{(l+2)}_T|\mathscr{G}|^{\mathcal{X}(l)}(t_i - t_{i-1})$ with a constant $ C$ depending only on $ b,\sigma,T$.  As a result, we have
    \begin{equation}
        \bigg|\E\left[ f(\X_T) - f(\aX_T)\right]\bigg| \le \sum_{i = 1}^n\left|\E\left[ e_i\right]\right| \le \sum_{i = 1}^n C \|v\|^{(l+2)}_T|\mathscr{G}|^{\mathcal{X}(l)}(t_i - t_{i -1}) \le CT\|f\|^{(l+2)}|\mathscr{G}|^{\mathcal{X}(l)}.
    \end{equation}
    This proves the desired result.
\end{proof}

\subsubsection{Proof of Theorem \ref{thm:concentration inequality}}\label{sec proof: thm conentration inequality}

Recall $\cF^{\xi} := \sigma\{ \xi_{i} \mid {i \ge 0}\}$ as defined in Section \ref{sec:conditional_weak_error}.
To facilitate the proof,
we define
 $\cF^{\xi}_i := \sigma\{ \xi_j \mid  j\le i-1\}$ for all $i\ge 0$, and define  the   filtration 
$\{\widetilde{\mathcal{F}}_t\}_{t\ge 0}$
with 
$  \widetilde{\mathcal{F}}_t:= \mathcal{F}_t^W \vee \mathcal{F}^{\xi} $. Note that $\{W_s\}_{s\ge 0}$ is still a Brownian motion on \((\Omega, \mathcal{F}, {\PP})\) with filtration \(\{\widetilde{\mathcal{F}}_t\}_{t \geq 0}\) and we have $\Ew\left[ \cdot\right] = \E\big[ \cdot \big|\widetilde{\mathcal{F}}_0\big]$.

The following lemma shows that $\X$
has finite moments conditional on \( \mathcal{F}^{\xi} \).
The proof is given in 
Appendix \ref{sec proof: lemma}.

\begin{Lemma}\label{lemma:moments of X under Pw}
    Under condition \eqref{eq:b upper bound uniform a} of (H.\ref{assump:linear growth on x}), for $ p\ge 2$, there is a constant $C_p$ such that 
    \begin{equation}\label{eq:moments of X under Ew}
        \Ew\left[ |\X_s|^{p}\right] \le C_p (1+ |x_0|^{p})e^{C_ps},
        \quad s\in [0,T].
    \end{equation}
\end{Lemma}

Based on Lemma \ref{lemma:moments of X under Pw}, we prove Theorem \ref{thm:concentration inequality}.
\begin{proof}[Proof of Theorem \ref{thm:concentration inequality}]
    Without loss of generality, assume that $t = t_j$ for some $1\le j \le n$. To prove \eqref{eq: concentration ineq}, consider the following PDE:
    \begin{equation}\label{eq:Kol_pde_j}
            \frac{\partial}{\partial t} v(t,x) + \mathcal{L} v(t,x) = 0,\quad (t,x)\in [0,t_j]\times \sR^d; \quad
            v(t_j,x) = f(x),\quad x\in \sR^d.
    \end{equation}
    Under the conditions of this theorem, the solution $v$ of~\eqref{eq:Kol_pde_j} is smooth and bounded. By the Feynman-Kac formula, $\Ew[f(\aX_{t_j})] = v(0,x_0)$. With this relation, we have
    \begin{align*}
        \Ew\left[ f(\X_{t_j}) - f(\aX_{t_j})\right] &= \Ew\left[ v(t_j,\X_{t_j}) - v(0,x_0)\right]\notag\\
        & = \Ew\left[ \sum_{i = 1}^j\left(v(t_i,\X_{t_i}) - v(t_{i-1},\X_{t_{i-1}})\right)\right] 
        =: \sum_{i = 1}^j\Ew\left[ e_i\right].\notag
    \end{align*}
    As in the proof of Theorem~\ref{thm:weak convergence}, we can decompose $\Ew\left[ e_i\right]$ into two terms and obtain
    \begin{equation}\label{mideq:decomp of diff}
    \Ew\left[ f(\X_T) \right] - \Ew\left[ f(\aX_T)\right] =\sum_{i=1}^j A_i + \sum_{i=1}^j B_i,
    \end{equation}
    where  $A_i$ is defined by
    \begin{equation}\label{mideq:martingale_diff}
        \begin{aligned}
            A_i := &\Ew\bigg[ \inti \left(\partial_t v(t_{i-1},\X_{t_{i-1}})  + b^{\top}(t_{i-1},\X_{t_{i-1}},\act_{t_{i-1}})\partial_x v(t_{i-1},\X_{t_{i-1}}) \right)\dd s 
        \notag\\
        & \quad + \inti \frac{1}{2} \tr\left(\sigma\sigma^{\top}(t_{i-1},\X_{t_{i-1}},\act_{t_{i-1}})\partial_x^2v(t_{i-1},\X_{t_{i-1}})\right) \dd s\bigg] \notag.
        \end{aligned}
    \end{equation}
    and $B_i$ is defined by
\change{\begin{equation}\label{mideq:B_drift}
        \begin{aligned}
            B_i &\coloneqq   \Ew\bigg[ \inti\left(\partial_tv(s,\X_s) - \partial_t v(t_{i-1},\X_{t_{i-1}})\right) \dd s\bigg]\\
         +&\Ew\bigg[\inti b^{\top}(s,\X_{s},\act_{t_{i-1}})\partial_x v(s,\X_s) - b^{\top}(t_{i-1},\X_{t_{i-1}},\act_{t_{i-1}})\partial_x v(t_{i-1},\X_{t_{i-1}}) \dd s\bigg]\\
        +&\Ew\bigg[\inti \frac{1}{2}\tr\left(\sigma\sigma^{\top}(s,\X_s,\act_{t_{i-1}})\partial_x^2 v(s,\X_s) - \sigma\sigma^{\top}(t_{i-1},\X_{t_{i-1}},\act_{t_{i-1}})\partial_x^2 v(t_{i-1},\X_{t_{i-1}}) \right)\dd s\bigg].
        \end{aligned}
    \end{equation}}
    As in Lemma \ref{lemma:convergence rate of diff}, apply It\^o's lemma to every \change{difference term on the right hand side of \eqref{mideq:B_drift}} and note that \change{ the resulting It\^o's integrals have mean zero under $\PP^W$}. Thus by   (H.\ref{assump:solution_kol_pde}) and (H.\ref{assump:linear growth on x}), the integrands of  \eqref{mideq:B_drift} are bounded by $C_p\|v\|_{C_p^4}(1+|\X_s|^{m_p})(t_i - t_{i-1})$ for some $m_p$ and $C_p$ that depend on $p$. This along with  Lemma~\ref{lemma:moments of X under Pw} implies 
    $ 
        |B_i| \le C_p\|v\|_{C_p^4}(t_i - t_{i-1})^2,
   $ 
    where $C_p$ is a constant depending on $b,\sigma,\bpi,T,p$. As a result, 
    \begin{equation}\label{mideq:B result}
        \sum_{i = 1}^j |B_i| \le \sum_{i = 1}^j C_p\|v\|_{C_p^4}|\mathscr{G}|(t_i - t_{i-1}) \le TC_p\|v\|_{C_p^4}|\mathscr{G}|.
    \end{equation}
    
    As for $A_i$,  note that $\{(A_i,\cF^{\xi}_i)\}_i $ is a martingale difference sequence. In fact, denote 
    \begin{equation}
    \begin{aligned}
        S_i := \bigg( \partial_t v(t_{i-1},\X_{t_{i-1}})  &+ b^{\top}(t_{i-1},\X_{t_{i-1}},\act_{t_{i-1}})\partial_x v(t_{i-1},\X_{t_{i-1}}) \\
        &+ \frac{1}{2} \tr\left(\sigma\sigma^{\top}(t_{i-1},\X_{t_{i-1}},\act_{t_{i-1}})\partial_x^2v(t_{i-1},\X_{t_{i-1}})\right)\bigg)(t_i-t_{i-1}).
    \end{aligned}
    \end{equation}
    Observing that $S_i$ is independent of $\xi_i,\xi_{i+1},\dots$, we have $A_i = \Ew\left[ S_i\right] = \E[ S_i\mid \cF^{\xi}_i]$ is $\cF^{\xi}_i$-measurable and for $i\ge 2$,
    \begin{equation}
        \begin{aligned}\label{mideq:expectation A = 0}
            \E[A_i\mid \cF^{\xi}_{i-1}] &= \E\left[ \Ew\left[S_i\right]\big|\xi_0,\dots,\xi_{i-2}\right]= \E\left[ \E\left[ S_i \big| \xi_0,\dots,\xi_{i-2},\xi_{i-1}\right]\big| \xi_0,\dots,\xi_{i-2}\right]\\
            & = \E\left[ \E\left[ S_i\big| X_{t_{i-1}},\xi_0,\dots,\xi_{i-2}\right]\big|\xi_0,\dots,\xi_{i-2}\right] = 0,
        \end{aligned}
    \end{equation}
    where the last equality follows from the same result in~\eqref{mideq:by pde equals to 0}. 

    By conditions (H.\ref{assump:solution_kol_pde}) and (H.\ref{assump:linear growth on x}), there is a constant $C_p$ and an integer $m_p$ such that $|S_i| \le C_p\|v\|_{C_p^4}(1+|\X_{t_{i-1}}|^{m_p})$. Again by  Lemma \ref{lemma:moments of X under Pw},  
        $|A_i| \le C_p\|v\|_{C_p^4}(t_i - t_{i-1})$. 
    By the Azuma–Hoeffding inequality \cite[Corollary 2.20]{wainwright2019highdimensional},   \change{for all $t\ge 0$},
    \begin{equation}\label{eq: concentration t}
        \PP\bigg( \bigg| \sum_{i=1}^j A_i\bigg| \ge t\bigg) \le 2\exp\left(-\frac{t^2}{4T(C_p\|v\|_{C_p^4})^2|\mathscr{G}|}\right).
    \end{equation}
    \change{Hence given $\rho \in (0,1)$, choosing  $ t \ge 0$ such that} $\varrho = 2\exp\left(-\frac{t^2}{4T(C_p\|v\|_{C_p^4})^2|\mathscr{G}|}\right)$ yields
    \begin{equation}
         \PP\bigg( \bigg| \sum_{i=1}^j A_i\bigg| \ge 2\sqrt{T}C_p\|v\|_{C_p^4}\sqrt{|\mathscr{G}|}\sqrt{\ln (2/\varrho)}\bigg) \le \varrho.
    \end{equation}
    It follows from~\eqref{mideq:B result} that
    \begin{equation}\notag
    \begin{aligned}
        \PP&\left( \left| \Ew[f(\X_{t_j})] - \Ew[f(\aX_{t_j})]\right| > (2+\sqrt{T})\sqrt{T}C_p\|v\|_{C_p^4}\sqrt{|\mathscr{G}|}\big(1+\sqrt{\ln (2/\varrho)}\big)\right) \\
        &\quad \le  \PP\bigg( \bigg| \sum_{i=1}^j A_i\bigg| \ge 2\sqrt{T}C_p\|v\|_{C_p^4}\sqrt{|\mathscr{G}|}\sqrt{\ln (2/\varrho)}\bigg) + \PP\left( \sum_{i=1}^j |B_i| \ge (2+\sqrt{T})\sqrt{T}C_p\|v\|_{C_p^4}\sqrt{|\mathscr{G}|}\right)\\
        &\quad \le \varrho + 0 = \varrho.
    \end{aligned}
    \end{equation}
    The desired result follows from $\|v\|_{C_p^4} \le C\|f\|_{C_p^4}$.

   To prove \eqref{eq: concentration ineq relax regularity}, we    start from \eqref{mideq:decomp of diff}. For \eqref{mideq:B_drift}, note that under assumption (H.\ref{assump:solution Kol pde relax regularity}), Lemma \ref{lemma:moments of X under Pw} remains valid, and the expectation in Lemma \ref{lemma:convergence rate of diff relax regularity} can be replaced by $\E^W$. Based on these observations, one can use a similar method in the proof of Theorem \ref{thm:weak convergence relax regularity} to obtain 
   $$
        |B_i| \le C\|v\|^{(l+2)}_{t_j}|\mathscr{G}|^{\mathcal{X}(l)}(t_i - t_{i-1}),
   $$
    which implies that
    \begin{equation}\label{eq: B term less regular}
        \sum_{i = 1}^j|B_i| \le CT\|v\|^{(l+2)}_{t_j}|\mathscr{G}|^{\mathcal{X}(l)}.
    \end{equation}
   Note that $\{A_i\}_i$ is still a martingale difference sequence. By (H.\ref{assump:solution Kol pde relax regularity}), we have 
    $$ 
        |A_i| \le C\|v\|_{t_j}^{(l+2)}(t_i - t_{i-1}).
    $$ 
    Again by the Azuma–Hoeffding inequality \cite[Corollary 2.20]{wainwright2019highdimensional}, \change{ for all $t \ge 0$,}
    $$
        \PP\bigg( \bigg| \sum_{i=1}^j A_i\bigg| \ge t\bigg) \le 2\exp\left({-\frac{t^2}{4T(C\|v\|_{t_j}^{(l+2)})^2|\mathscr{G}|}}\right).
    $$
    \change{Hence given $\varrho \in (0,1)$, choosing $t \ge 0$ such that} $\varrho = 2\exp\left({-\frac{t^2}{4T(C\|v\|_{t_j}^{(l+2)})^2|\mathscr{G}|}}\right)$ yields
    $$
         \PP\bigg( \bigg| \sum_{i=1}^j A_i\bigg| \ge 2\sqrt{T}C\|v\|_{t_j}^{(l+2)}\sqrt{|\mathscr{G}|}\sqrt{\ln (2/\varrho)}\bigg) \le \varrho.
    $$
    It follows from~\eqref{eq: B term less regular} that
    \begin{equation*} 
    \begin{aligned}
        \PP&\left( \left| \Ew[f(\X_{t_j})] - \Ew[f(\aX_{t_j})]\right| > (2+\sqrt{T})\sqrt{T}C\|v\|_{t_j}^{(l+2)}|\mathscr{G}|^{\mathcal{X}(l)\land \frac{1}{2}}\big(1+\sqrt{\ln (2/\varrho)}\big)\right) \\
        &\quad \le  \PP\bigg( \bigg| \sum_{i=1}^j A_i\bigg| \ge 2\sqrt{T}C\|v\|_{t_j}^{(l+2)}\sqrt{|\mathscr{G}|}\sqrt{\ln (2/\varrho)}\bigg) + \PP\left( \sum_{i=1}^j |B_i| \ge (2+\sqrt{T})\sqrt{T}C\|v\|_{t_j}^{(l+2)}|\mathscr{G}|^{\mathcal{X}(l)}\right)\\
        &\quad \le \varrho + 0 = \varrho.
    \end{aligned}
    \end{equation*}
    The desired result follows from $\|v\|_{t_j}^{(l+2)} \le C\|f\|^{(l+2)}$.
\end{proof}

\subsubsection{Proof of Theorem \ref{thm:strong conv}}\label{sec proof: thm strong conv}
We introduce the Orlicz norm on $ (\Omega^{\xi},\mathcal{F}^{\xi},\PP^{\xi})$. Define $ \psi_1(x)\coloneqq   e^x - 1$ for all $x>0$, and   for a random variable $ X: \Omega^{\xi} \rightarrow \R^m, m\in \N_{+}$, define the Orlicz norm  
\begin{equation}
    \|X\|_{\psi_1} \coloneqq  \inf\left\{ C>0: \E\left[ \psi_1\left(\frac{|X|}{C}\right)\right] \le 1\right\},
\end{equation}
where $|\cdot|$ is the Euclidean  norm on $\R^m$. By   \cite[Proposition 2.7.1]{Vershynin2018high}, $ X$ is sub-exponential if and only if $ \|X\|_{\psi_1} <\infty$.

The following lemma 
follows directly from \cite[Lemma 2.2.2]{van_der_vaart1996}
and the fact that 
$ 
    \|X\|_p \le \sqrt[p]{p!}\|X\|_{\psi_1}
$, due to the inequality that  $ |x|^p/p! \le e^{|x|} - 1$ for all $x\in \sR^m$. 
\begin{Lemma}\label{lemma: max p norm bound strong conv}
  There is a constant
    $ C\ge 0$ such that for all sub-exponential random variables  $ \{X_i\}_{i=1}^n$,
    \begin{equation}
        \big\|\max_{1\le i\le n}|X_i|\big\|_{\psi_1} \le C\log(1+n)\max_{1 \le i \le n}\|X_i\|_{\psi_1}.
    \end{equation}
    Thus, for $ p \ge 1$, there is a constant $ C_p\ge 0$  such that
    \begin{equation}
        \E\left[ \max_{1\le i \le n} |X_i|^p\right] \le C_p\log^p(1+n)\max_{1 \le i \le n}\|X_i\|_{\psi_1}^p.
    \end{equation}
\end{Lemma}

Note that under condition (H.\ref{assum:lipschitz strong conv}), $\tilde b^{\bpi}$ and $ \tilde \sigma^{\bpi}$ is globally Lipschitz with respect to the state $ x$. By classical SDE results, e.g. \cite[Theorem 7.9]{graham2013}, we have the following lemma.
\begin{Lemma}\label{lemma:monments strong conv}
    Assume (H.\ref{assum:lipschitz strong conv}) holds. For every $ p \ge 2$, there is a constant $ C_p$, such that 
    \begin{equation}
        \Ew\left[ \sup_{0\le t\le T}|\aX_t|^p\right] \le C_p, \quad \E\left[ \sup_{0\le t\le T}|\X_t|^p\right] \le C_p.
    \end{equation}
    Moreover,
    \begin{equation}
        \Ew\left[ \sup_{0\le s\le T} |\aX_s - \aX_{\delta(s)}|^p\right] \le C_p|\mathscr{G}|^{\frac{p}{2}}, \quad \E\left[ \sup_{0\le s\le T}|\X_s - \X_{\delta(s)}|^p\right] \le C_p|\mathscr{G}|^{\frac{p}{2}}.
    \end{equation}
\end{Lemma}

\begin{proof}[Proof of Theorem \ref{thm:strong conv}]
By
(H.\ref{assum:lipschitz strong conv}),
$d_A(\phi(0,0,\xi),a_0)$ is sub-exponential.  Recall the notation $ i(t) := \max\{i: t_i \le t\}$ and for all $t\in [0,T]$,
\begin{equation}
\begin{aligned}
    \X_t &= x_0 + \int_0^t b(s,\X_s,\phi(\delta(s),\X_{\delta(s)},\xi_{i(s)})) \dd s + \int_0^t \sigma(s,\X_s)\dd W_s,\\
    \aX_t &=  x_0 + \int_0^t \tilde b^{\bpi}(s,\aX_s) \dd s + \int_0^t \sigma(s,\aX_s)\dd W_s.
\end{aligned}
\end{equation}
For all $t\in [0,T]$,   $ \Delta X_t := \X_t - \aX_t$ satisfies 
\begin{equation}
\label{eq:Delta_X_strong_conv}
    \begin{aligned}
    &\sup_{0\le u\le t}|\Delta X_u| 
    \\
    &= \sup_{0\le u\le t}\bigg|\int_0^u \left(b(s,\X_s,\phi(\delta(s),\X_{\delta(s)},\xi_{i(s)})) - b(s,\aX_s,\phi(s,\aX_{s},\xi_{i(s)}))\right)\dd s  \\
    &\quad + \int_0^u \left(b(s,\aX_s,\phi(s,\aX_{s},\xi_{i(s)})) - \tilde b^{\bpi}(s,\aX_s)\right) \dd s + \int_0^u \left(\sigma(s,\X_s) - \sigma(s,\aX_s) \right)\dd W_s \bigg|.
    \end{aligned}
\end{equation}
By the Lipschitz continuity of $\sigma$ and the Burkholder-Davis-Gundy inequality,   for all $ p \ge 2$,
\begin{equation}\notag
\begin{aligned}
    \E\left[ \sup_{0\le u\le t}\left| \int_0^u \left( \sigma(s,\X_s) - \sigma(s,\aX_s)\right)\dd W_s\right|^p\right] &\le C_p\E\left[ \left(\int_0^t\left|\sigma(s,\X_s) - \sigma(s,\aX_s) \right|^2\dd s\right)^{\frac{p}{2}}\right]\\
    &\le C_p\E\left[ \left( \int_0^t |\Delta X_u|^2\dd u\right)^{\frac{p}{2}}\right] \le C_p\E\left[ \int_0^t |\Delta X_u|^p\dd u\right],
    \end{aligned}
\end{equation}
where $ C_p>0 $ is a generic constant that depends on 
$p$
 and the coefficients, but is independent of the grid
$\mathscr{G}$.
This 
  along with \eqref{eq:Delta_X_strong_conv} and
the regularity  of $b$ and $\phi$ in (H.\ref{assum:lipschitz strong conv}) implies that 
\begin{equation}\label{eq:delta X decomp strong conv}
    \begin{aligned}
        \E\left[ \sup_{0\le s\le t}\left| \Delta X_s\right|^p\right]\le& C_p|\mathscr{G}|^{\frac{p}{2}} + C_p \int_0^t \E\left[\left| \Delta X_s\right|^p+ \left| \X_s - \X_{\delta(s)}\right|^p \right] \dd s\\
        &+ C_p\E\left[ \sup_{0\le u \le t}\left| \int_{0}^{u}\left(b(s,\aX_s,\phi(s,\aX_s,\xi_{i(s)})) - \tilde b^{\bpi}(s,\aX_s) \right)\dd s \right|^p\right].
    \end{aligned}
\end{equation}

We now bound the last term in \eqref{eq:delta X decomp strong conv}. For all $s\in [0,T]$, write $ U_s := b(s,\aX_s,\phi(s,\aX_s,\xi_{i(s)})) - \tilde b^{\bpi}(s,\aX_s)$. By (H.\ref{assum:lipschitz strong conv}), we have
\begin{equation}\label{mideq:U bound strong conv}
|U_s| = \left|b(s,\aX_s,\phi(s,\aX_s,\xi_{i(s)})) - \tilde b^{\bpi}(s,\aX_s) \right| \le C\left(1+\sup_{0\le t\le T} |\aX_t| + d_A(\phi(0,0,\xi_{i(s)}),a_0)\right). 
\end{equation}
Denote $ \E^\xi[\cdot] \coloneqq \E\left[ \cdot | \mathcal{F}^W\right]$ and note that $ \xi$ is independent of $ \mathcal{F}^W$. 
\begin{equation}\label{mideq:decomp strong conv}
    \E\left[ \sup_{0\le u\le t} \left| \int_0^u U_s \dd s\right|^p\right] \le C_p\E\left[ \sup_{0\le u\le t}\left| \int_0^{\delta(u)}U_s \dd s\right|^p\right] + C_p \E\left[\E^\xi\left[\sup_{0\le u\le t}\left| \int_{\delta(u)}^u U_s \dd s\right|^p\right]\right].
\end{equation}
By \eqref{mideq:U bound strong conv} and Lemma \ref{lemma: max p norm bound strong conv},
the second term on the right hand side of \eqref{mideq:decomp strong conv}
satisfies 
\begin{equation}\label{mideq:time diff int bound strong conv}
\begin{aligned}
    \E\left[\E^\xi\left[\sup_{0\le u\le t}\left| \int_{\delta(u)}^u U_s \dd s\right|^p\right]\right] &\le C_p\E\left[\left( 1 + \sup_{0\le t\le T}\left| \aX_t\right|^p + \E\left[ \max_{1\le j\le n}\left( d_A(\phi(0,0,\xi_j),a_0)\right)^p\right]\right)|\mathscr{G}|^p\right]\\
    &\le C_p\E\left[\left( 1+ \sup_{0\le t\le T}\left| \aX_t\right|^p + \log^p(1+n)\|d_A(\phi(0,0,\xi),a_0)\|_{\psi_1}^p\right)|\mathscr{G}|^p\right]\\
    &\le C_p|\mathscr{G}|^{\frac{p}{2}},
\end{aligned}
\end{equation}
where  the last inequality  used the condition  $|\mathscr{G}|^{1/2}\log(1+n) \le C_G$. 

To bound  the first term on the right hand side of \eqref{mideq:decomp strong conv}, for each $i=0,\ldots n-1$, let  $ M_i  := \int_{t_i}^{t_{i+1}} U_s \dd s$ and $ \mathcal{F}_i := \sigma(\xi_j: 0\le j\le i-1) \lor \mathcal{F}^W$. Since $ \left\{ (\sum_{j=0}^{i-1}M_j,\mathcal{F}_i)\right\}_{i=1}^n$ is a martingale, by  Doob's inequality,
\begin{equation}
\label{eq:randominization_error_strong}
    \E\left[ \sup_{0\le u\le t}\left| \int_0^{\delta(u)}U_s \dd s\right|^p\right] = \E\left[ \max_{1\le i\le n}\Big|\sum_{j=0}^{i-1}M_j\Big|^p\right] \le C_p \E\left[ \left|\sum_{j=0}^{n-1}M_j\right|^p\right].
\end{equation}
In the sequel, 
we  estimate  \eqref{eq:randominization_error_strong} by quantifying the tail behavior of 
$\left|\sum_{j=0}^{n-1}M_j\right|$
  conditioning on the realization of the Brownian motion $W$.

{Let $ \|\cdot\|_{\psi_1}$ be the Orlicz norm on $ (\Omega^{\xi},\mathcal{F}^{\xi},\PP^{\xi})$, conditioning on $ \mathcal{F}^W$}. 
  Since for all $j=0,\ldots, i-1$,
  $$ \E^\xi[M_j]=\int_{t_j}^{t_{j+1}} \E^\xi [U_s] \dd s
  =\int_{t_j}^{t_{j+1}} \E^\xi [ b(s,\aX_s,\phi(s,\aX_s,\xi_{i(s)})) - \tilde b^{\bpi}(s,\aX_s)]\dd s
  =0,
  $$
and hence by   \cite[Theorem 6.21]{ledoux1991}, 
\begin{equation*}
\begin{aligned}
    \Big\|\sum_{j=0}^{n-1} M_j\Big\|_{\psi_1} 
    \le C\left( \E^\xi\left[\Big|\sum_{j=0}^{n-1} M_j\Big|\right] + \|\max_j|M_j|\|_{\psi_1}\right).
    \end{aligned}
\end{equation*}
Note that conditioning on $\cF^W$, $\xi_i$ is the only randomness in $M_i$,
and hence  by the independent of $\xi$ and $\xi_j$ for $i\not =j$, 
$ \E^\xi\left[ M_i^{\top}M_j\right]
= \E^\xi [ M_i]^{\top}\E^\xi 
[  M_j ] = 0$ for all $i\ne j$. Then using  \eqref{mideq:U bound strong conv}  and the square integrability of $ \phi(0,0,\xi)$, 
\begin{equation*}
    \E^\xi\bigg[\Big|\sum_{j=0}^{n-1} M_j\Big|\bigg] \le \sqrt{\E^\xi\bigg[\Big|\sum_{j=0}^{n-1} M_j\Big|^2\bigg]} = \sqrt{\sum_{j=0}^{n-1}\E^\xi\left[ |M_j|^2\right]} \le C\left(1+ \sup_{0\le t\le T}|\aX_t|\right)|\mathscr{G}|^{\frac{1}{2}}.
\end{equation*}
 By   \eqref{mideq:U bound strong conv} and (H.\ref{assum:lipschitz strong conv}),  
\begin{equation}\label{mideq: sub-exp norm bound strong conv}
    \|M_j\|_{\psi_1} \le C\left( 1+ \sup_{0\le t\le T}|\aX_t| \|1\|_{\psi_1} + \|d_A(\phi(0,0,\xi),a_0)\|_{\psi_1}\right)|\mathscr{G}| 
    \le C\left( 1+ \sup_{0\le t\le T}|\aX_t|   \right)|\mathscr{G}|,
\end{equation}
which along with Lemma \ref{lemma: max p norm bound strong conv} implies that 
$  \left\|\max_j|M_j|\right\|_{\psi_1}  \le C\log(1+n)\left( 1+\sup_{0\le t\le T} |\aX_t|\right)|\mathscr{G}|$,
and    
\begin{equation}\label{eq: M_j orlicz norm bound}
    \Big\|\sum_{j=0}^{n-1} M_j\Big\|_{\psi_1} \le C(1+ \sup_{0\le t\le T}|\aX_t|)|\mathscr{G}|^{\frac{1}{2}}.
\end{equation}
For all $t>0$,  
by the conditional Markov inequality and   $ \E^\xi[ \psi_1(|\sum_{j =1 }^{n-1} M_j|/\|\sum_{j =1 }^{n-1} M_j\|_{\psi_1})] \le 1$, 
\begin{equation*}
\begin{aligned}
    \PP^\xi\left( \bigg|\sum_{j=0}^{n-1}M_j\bigg|^p \ge t\right) &= \PP^\xi\left( \psi_1\left(\frac{\left|\sum_{j=0}^{n-1}M_j\right|}{\left\|\sum_{j=0}^{n-1} M_j\right\|_{\psi_1}} \right) +1 \le \psi_1\left(\frac{t^{\frac{1}{p}}}{\left\|\sum_{j=0}^{n-1} M_j\right\|_{\psi_1}}\right)+1\right)\\
    &\le 2\left(\psi_1\left(\frac{t^{\frac{1}{p}}}{\left\|\sum_{j=0}^{n-1} M_j\right\|_{\psi_1}}\right)+1\right)^{-1}\le 2\exp\left\{ -\frac{t^{\frac{1}{p}}}{C(1+\sup_{0\le t\le T}|\aX_t|)|\mathscr{G}|^{\frac{1}{2}}}\right\},
    \end{aligned}
\end{equation*}
 {where  the last  inequality used  that   $ (\psi_1(x)+1)^{-1} = \exp(-x) $ for all $x>0$, and \eqref{eq: M_j orlicz norm bound}}.
Thus 
\begin{equation}\label{eq:concen ineq expectation strong conv}
    \E^\xi\left[ \bigg|  {\sum_{j=0}^{n-1} M_j} \bigg|^p\right]
    =\int_0^\infty\mathbb{P}^\xi\left(\bigg|  {\sum_{j=0}^{n-1} M_j} \bigg|^p\ge  t\right)\dd t\le C_p\left(1+\sup_{0\le t\le T}|\aX_t|^p\right)|\mathscr{G}|^{\frac{p}{2}}.
\end{equation}

Combining \eqref{eq:delta X decomp strong conv}, \eqref{mideq:decomp strong conv}, \eqref{mideq:time diff int bound strong conv}, \change{\eqref{eq:randominization_error_strong}} and \eqref{eq:concen ineq expectation strong conv}, using Lemma \ref{lemma:monments strong conv} to bound the time difference of $ \X$ and $ \aX$ and the moments of $ \sup_t|\aX_t|$, we have 
\begin{equation}
\begin{aligned}
    \E\left[ \sup_{0\le s\le t}\left| \Delta X_s\right|^p\right] &\le C_p|\mathscr{G}|^{\frac{p}{2}} + C_p \int_0^t\left(\E\left[ \sup_{0\le u\le s}\left| \Delta X_u\right|^p\right] + \E\left[| \X_s - \X_{\delta(s)}|^p  \right]\right) \dd s \\
    &\quad + C_p|\mathscr{G}|^{\frac{p}{2}} \int_0^t \E\left[ 1+ \sup_{0\le t\le T}|\aX_t|^p\right]\dd s\\
    &\le C_p\left| \mathscr{G}\right|^{\frac{p}{2}} + C_p\int_0^t \E\left[ \sup_{0\le u\le s}|\Delta X_u|^p\right] \dd s,
\end{aligned}
\end{equation}
which yields the desired estimate 
$\E\left[ \sup_{0\le s\le T}\left| \Delta X_s\right|^p\right] \le C_p|\mathscr{G}|^{\frac{p}{2}}$
due to     Gr\"onwall's inequality.

It remains to   prove the almost sure convergence with  $|\mathscr{G}| = \frac{T}{n}$. For $p \ge 2$ and $t \ge 0$, 
\begin{equation}
\label{eq:almost_sure_P}
    \PP\left( \sup_{0\le t\le T}|\Delta X_t| > t\right) = \PP\left( \sup_{0\le t\le T}|\Delta X_t|^p > t^p\right) \le t^{-p}\E\left[ \sup_{0\le t\le T}|\Delta X_t|^p\right] \le C_p n^{-\frac{p}{2}} t^{-p}.
\end{equation}
Given $\varepsilon\in (0,1)$, setting $t = n^{-\frac{1}{2}+\varepsilon}$ in \eqref{eq:almost_sure_P} yields
for all $p\ge 2$,
$$ 
    \PP\left( \sup_{0\le t\le T}|\Delta X_t| > n^{-\frac{1}{2}+\varepsilon}\right) \le C_{p}n^{-\varepsilon p}.
$$ 
By choosing $p = \frac{1}{\varepsilon} +2 $, we have
$$ 
    \sum_{n=1}^{\infty} \PP\left( \sup_{0\le t\le T}|\Delta X_t| > n^{-\frac{1}{2}+\varepsilon}\right) \le C_p\sum_{n=1}^{\infty} n^{-1-2\varepsilon} < +\infty.
$$ 
The desired result follows from the Borel-Cantelli lemma.
\end{proof}

\change{
\subsection{Proof of Corollary \ref{cor:concentration inequality}}\label{sec proof: cor concen ineq}

\begin{proof}[Proof of Corollary \ref{cor:concentration inequality}]

Recall that $\PP^W$ is the conditional probability measure given $\mathcal{F}^\xi$. 
Under the   measure $\PP^W$,
$\{f(X^{\mathscr{G}_n}_T(W_k,\xi))\}_{k=1}^m$ are independent random variables taking values in $[-\|f\|_{\infty},\|f\|_{\infty}]$,
where $\|f\|_{\infty}$
is the sup-norm of $f$.
By 
Hoeffding's inequality,
  for all $t\ge 0$ and $m,n\in \sN$, 
\begin{equation}
    \PP^W\left( \left| \frac{1}{m}\sum_{k=1}^m \left(f\left(X^{\mathscr{G}_n}_T(W_k,\xi)\right) - \E^W\left[f\left(X^{\mathscr{G}_n}_T(W_k,\xi)\right)\right]
    \right)\right|\ge t\right)\le 2\exp\left( -\frac{2t^2m}{4 \|f\|_{\infty}^2}\right),
\end{equation}
For a given $\varrho \in (0,1)$, choosing $t$ such that $\varrho = 2\exp(-t^2m/(2\|f\|_{\infty}^2))$ yields
\begin{equation}
    \PP^W\left( \left|  \frac{1}{m}\sum_{i=1}^m f\left(X^{\mathscr{G}_n}_T(W_k,\xi)\right) - \E^W\left[f\left(X^{\mathscr{G}_n}_T(W_k,\xi)\right)\right]\right|\ge \sqrt{2} \|f\|_{\infty} \frac{1}{\sqrt{m}}\ln(2/\varrho)\right) \le \varrho.
\end{equation}
This   along with Theorem \ref{thm:concentration inequality} shows that 
there exists a constant $C\ge 0$
such that for all  $n,m\in \sN $ and  $\varrho \in (0,1) $,
    \begin{equation}
        \PP\left( \bigg|\hat I_{m,n} - \E[f(\aX_T)]\bigg|\ge C \left(1+\sqrt{\ln(2/\varrho)}\right)(n^{-1/2} + m^{-1/2})\right)\le \varrho.
\end{equation}
The desired result follows directly by choosing $n$ and $m$ as in the statement.
\end{proof}
}

\subsection{Proofs of Propositions
\ref{prop:value function}, 
\ref{prop: martingale orth no a application}
and 
 \ref{prop:quadratic variation application}}\label{sec proof: proposition}

\begin{proof}[Proof of Proposition \ref{prop:value function}]
   For the first part (\ref{enum:value function 1}), note that 
    \begin{equation}\label{mideq:value function diff}
        \begin{aligned}
            J^{\mathscr{G}} - \tilde J =\left(\E\left[ h(\X_T)\right] - \E\left[ h(\aX_T)\right]\right) + \sum_{i = 0}^{n-1} \E\left[ \int_{t_i}^{t_{i+1}} (r(s,\X_s,\act_{t_i}) - \tilde r(s,\aX_s)) \dd s\right].
        \end{aligned}
    \end{equation}
    By Theorem \ref{thm:weak convergence}, there is a constant $ C_p$, such that 
    \begin{equation}
        \left|\E\left[ h(\X_T)\right] - \E\left[ h(\aX_T)\right]\right| \le C_p\|h\|_{C_p^4}|\mathscr{G}|.
    \end{equation}
    For the second term on the right hand side of \eqref{mideq:value function diff}, note that if $ t_i\le s< t_{i+1}$, then 
    \begin{equation}\label{mideq:running term decomp}
        \begin{aligned}
            \E\left[r(s,\X_s,\act_{t_i}) - \tilde r(s,\aX_s)\right] &= \E\left[ r(s,\X_s,\act_{t_i}) - r(t_i,\X_{t_i},\act_{t_i})\right] + \E\left[ r(t_i,\X_{t_i},\act_{t_i}) - \tilde r(s,\aX_s)\right]\\
            & = \E\left[ r(s,\X_s,\act_{t_i}) - r(t_i,\X_{t_i},\act_{t_i})\right] + \E\left[ \tilde r(t_i,\X_{t_i}) - \tilde r(t_i,\aX_{t_i})\right] \\
            &\quad \quad + \E\left[ \tilde r(t_i,\aX_{t_i}) - \tilde r(s,\aX_s)\right].
        \end{aligned}
    \end{equation}
    Since $r \in C^{2,0}_p$, it follows from It\^o's lemma that 
    \begin{equation}
        |\E\left[ r(s,\X_s,\act_{t_i}) - r(t_i,\X_{t_i},\act_{t_i})\right]| \le C_p\|r\|_{C_p^{2,0}}|\mathscr{G}|.
    \end{equation}
    Set $ f = \tilde r(t_i,\cdot)$ in Theorem \ref{thm:weak convergence}, we have  
    \begin{equation}
        \left| \E\left[ \tilde r(t_i,\X_{t_i}) - \tilde r(t_i,\aX_{t_i})\right]\right| \le C_p\|\tilde r(t_i,\cdot)\|_{C_p^4}|\mathscr{G}|.
    \end{equation}
    Using It\^o's formula, we have
    \begin{equation}
        \left|\E\left[ \tilde r(t_i,\aX_{t_i}) - \tilde r(s,\aX_s)\right]\right| \le C_p\|\tilde r\|_{C_p^{2}}|\mathscr{G}|.
    \end{equation}
    Combining these inequalities, we prove the desired result.

   For the second part (\ref{enum:value function 2}), we start at \eqref{mideq:value function diff}. By Theorem \ref{thm:weak convergence relax regularity}, there is a constant $ C\ge 0$  such that 
    \begin{equation}
        \left|\E\left[ h(\X_T)\right] - \E\left[ h(\aX_T)\right]\right| \le C\|h\|^{(l+2)}|\mathscr{G}|^{\mathcal{X}(l)}.
    \end{equation}
For the second term on the right hand side of \eqref{mideq:value function diff}, we still have \eqref{mideq:running term decomp}.
   By Lemma \ref{lemma:convergence rate of diff relax regularity}, 
    \begin{equation}
        |\E\left[ r(s,\X_s,\act_{t_i}) - r(t_i,\X_{t_i},\act_{t_i})\right]| \le C\|r\|_T^{(l)}|\mathscr{G}|^{\mathcal{X}(l)}.
    \end{equation}
    Set $ f = \tilde r(t_i,\cdot)$ in Theorem \ref{thm:weak convergence relax regularity}, we have  
    \begin{equation}
        \left| \E\left[ \tilde r(t_i,\X_{t_i}) - \tilde r(t_i,\aX_{t_i})\right]\right| \le C\|\tilde r(t_i,\cdot)\|^{(l+2)}|\mathscr{G}|^{\mathcal{X}(l)} \le C\sup_{0\le t\le T}\|\tilde r(t,\cdot)\|^{(l+2)}|\mathscr{G}|^{\mathcal{X}(l)}.
    \end{equation}
    Note that $\tilde r \in \mathcal{H}^{(l)}_T$, by   \cite[Lemma 14.1.6]{kloeden1992stochastic},  
   $     \left|\E\left[ \tilde r(t_i,\aX_{t_i}) - \tilde r(s,\aX_s)\right]\right| \le C\|\tilde r\|_T^{(l)}|\mathscr{G}|^{\mathcal{X}(l)}.
   $ 
    Combining these inequalities yield the desired result. 
\end{proof}

\begin{proof}[Proof of Proposition \ref{prop: martingale orth no a application}]
For the variance, note that
\begin{equation}
    \begin{aligned}
        &\E\left[ (S\circ V)^2\right]\\
        =& \E\bigg[\bigg(\int_0^T S(\delta(t),\X_{\delta(t)},\act_{\delta(t)}) g(t,\X_t,\act_{\delta(t)})\dd t + \int_0^T S(\delta(t),\X_{\delta(t)},\act_{\delta(t)})(\partial_x V(t,\X_t))^{\top}\sigma(t,\X_t,\act_{\delta(t)})\dd W_t\bigg)^2\bigg]\\
        \le&2 \int_{[0,T]^2}\E\left[S(\delta(t),\X_{\delta(t)},\act_{\delta(t)}) g(t,\X_t,\act_{\delta(t)})S(\delta(s),\X_{\delta(s)},\act_{\delta(s)}) g(s,\X_s,\act_{\delta(s)})\right]\dd s\dd t\\
        &+2 \int_0^T \E\left[\left| S(\delta(t),\X_{\delta(t)},\act_{\delta(t)})(\partial_x V(t,\X_t))^{\top}\sigma(t,\X_t,\act_{\delta(t)})\right|^2 \right]\dd t.
    \end{aligned}
\end{equation}
The desired result follows from Lemma \ref{lemma:moments of X}.

For the bias part, it follows from It\^o's formula that 
\begin{equation}\label{eq:orth equality}
    \begin{aligned}
        &\Ew\left[ \int_0^T S(t,\aX_t)\left(\dd V(t,\aX_t) + \tilde r(t,\aX_t)\dd t\right)\right] \\
        &= \Ew\left[ \int_0^T S(t,\aX_t)\left(\partial_t V + (\tilde b^{\bpi})^{\top}\partial_x V + \frac{1}{2}\tr\left((\tilde \sigma^{\bpi})^2 \partial_x^2 V\right) + \tilde r\right)(t,\aX_t)\dd t\right]
    \end{aligned}
\end{equation}
and 
\begin{equation}\label{eq:sampled orth equality}
\begin{aligned}
    \E[(S,V)] &= \sum_{i=0}^{n-1}\E\left[S(t_{i},\X_{t_{i}})\left( V(t_{i+1},\X_{t_{i+1}}) - V(t_i,\X_{t_i}) + R^{\mathscr{G}}_{t_i}\right) \right]\\
    &= \E\bigg[\int_0^T S(\delta(t),\X_{\delta(t)})g(t,\X_t,a_{\delta(t)})\dd t\bigg]\\
    & = \E\left[ \int_0^T S(t,\X_t)g(t,\X_t,a_{\delta(t)})\dd t\right] - \E\left[ \int_0^T \Delta S_tg(t,\X_t,a_{\delta(t)})\dd t\right],
    \end{aligned}
\end{equation}
where $ \Delta S_t := S(t,\X_t) - S(\delta(t),\X_{\delta(t)})$. We have 
\begin{equation}
\begin{aligned}
        &\left|\E[S \circ V ] - \Ew\left[ \int_0^T S(t,\aX_t)\left(\dd V(t,\aX_t) + \tilde r(t,\aX_t)\dd t\right)\right] \right| \\
        &\le \left| \E\left[\int_0^TSg(t,\X_t,a_{\delta(t)})\dd t\right] - \E\left[ \int_0^T\int_A Sg(t,\aX_t,a)\bpi(\dd a|t,\aX_t)\dd t\right]\right| + \left| \E\left[ \int_0^T \Delta S_tg(t,\X_t,a_{\delta(t)})\dd t\right]\right|.
        \end{aligned}
\end{equation} 

The second term on the right hand side of above equation is bounded by $ C_p\|S\|_{C_p^{2,0}}\|g\|_{C_p^{0,0}}|\mathscr{G}|$ and $ C\|S\|_T^{(l)}\|g\|_T^{(0)}|\mathscr{G}|^{\mathcal{X}(l)}$ using the methods of Theorem \ref{thm:weak convergence} and Theorem \ref{thm:weak convergence relax regularity}, respectively. The desired results follow from the application of Proposition \ref{prop:value function}.  
\end{proof}

\begin{proof}[Proof of Proposition \ref{prop:quadratic variation application}]
For the variance part, note that
\begin{equation}\notag
    \begin{aligned}
        &\var{S\Box V}  \le \E\left[ (S\Box V)^2\right] \\
        &= \E\left[ \left( \sum_{i=1}^{n} S(t_{i-1},\X_{t_{i-1}},a_{t_i})\left( \int_{t_{i-1}}^{t_{i}} g_1(s,\X_s,\act_{t_{i-1}})\dd s + \int_{t_{i-1}}^{t_i}(\partial_x V(s,\X_s))^{\top}\sigma(s,\X_s,\act_{t_{i-1}}) \dd W_s\right)^2\right)^2\right]\\
        &\le C_p \|S\|_{C_p^{0,0}}^2\left( \|g_1\|_{C_p^{0,0}}^2 + \|V\|_{C_p^2}^2 \right)^2.
    \end{aligned}
\end{equation}
For the bias part, it follows from It\^o's formula that
    \begin{equation}
\begin{aligned}
    S\Box V &=\sum_{i=1}^{n} S(t_{i-1},\X_{t_{i-1}},a_{t_i})\left( V(t_i,\X_{t_i}) - V(t_{i-1},\X_{t_{i-1}}) + R^{\mathscr{G}}_{t_{i-1}}\right)^2 \\
    &= \sum_{i=1}^{n} S(t_{i-1},\X_{t_{i-1}},a_{t_i})\left( \int_{t_{i-1}}^{t_{i}} g_1(s,\X_s,\act_{t_{i-1}})\dd s + \int_{t_{i-1}}^{t_i}(\partial_x V(s,\X_s))^{\top}\sigma(s,\X_s,\act_{t_{i-1}}) \dd W_s\right)^2\\
    &= \sum_{i=1}^{n} S(t_{i-1},\X_{t_{i-1}},a_{t_i}) \bigg(\bigg( \int_{t_{i-1}}^{t_{i}} g_1(s,\X_s,\act_{t_{i-1}})\dd s\bigg)^2 \\
    &\quad+ 2\int_{t_{i-1}}^{t_{i}} \big(g_1(s,\X_s,\act_{t_{i-1}}) - g_1(t_{i-1},\X_{t_{i-1}},\act_{t_{i-1}})\big)\dd s \int_{t_{i-1}}^{t_i}(\partial_x V(s,\X_s))^{\top}\sigma(s,\X_s,\act_{t_{i-1}}) \dd W_s \\
    &\quad+ 2\Delta t_i g(t_{i-1},\X_{t_{i-1}},\act_{t_{i-1}})\int_{t_{i-1}}^{t_i}(\partial_x V(s,\X_s))^{\top}\sigma(s,\X_s,\act_{t_{i-1}}) \dd W_s\\
    &\quad+ \bigg( \int_{t_{i-1}}^{t_i}(\partial_x V(s,\X_s))^{\top}\sigma(s,\X_s,\act_{t_{i-1}}) \dd W_s\bigg)^2\bigg).
    \end{aligned}
\end{equation}
Note that the first and second terms are bounded by $ C_p(\|S\|_{C_p^{2,0}}\|g_1\|_{C_p^{2,0}}^2+\|g_1\|_{C_p^{2,0}}\|V\|_{C_p^2})|\mathscr{G}|$ and $C(\|S\|_T^{(l)}(\|g_1\|_T^{(l)})^2+\|g_1\|_T^{(l)}\|V\|_T^{(l)})|\mathscr{G}|^{\mathcal{X}(l)}$ under different types of conditions,
and third term is mean zero. For the last term, we have 
\begin{equation}\label{eq:sampled quadratic variation}
\begin{aligned}
    &\E\left[ \sum_{i=0}^{n-1} S(t_{i-1},\X_{t_{i-1}},a_{t_{i-1}})\left( \int_{t_{i-1}}^{t_i} (\partial_xV(s,\X_s))^{\top}\sigma(s,\X_s,\act_{t_{i-1}})\dd W_s\right)^2\right]\\
    &= \E\left[ \int_0^T S(\delta(t),\X_{\delta(t)},a_{\delta(t)}) g_2(t,\X_t,a_{\delta(t)}) \dd t \right].
    \end{aligned}
\end{equation}
It converges to $\E\left[ \int_0^T  \widetilde{Sg}_2(t,\aX_t)\dd t \right] $ with error bounds $C_p(\|S\|_{C_p^{2,0}}\|g_2\|_{C_p^{0,0}}+ \|Sg_2\|_{C_p^{2,0}}+ \|\widetilde{Sg_2}\|_{C_p^2} + \sup_{t\in[0,T]}\|\widetilde{Sg_2}(t,\cdot)\|_{C_p^4})|\mathscr{G}| $ and $ C(\|S\|_T^{(l)}\|g_2\|_T^{(0)}+ \|Sg_2\|_T^{(l)}+ \|\widetilde{Sg_2}\|_T^{(l)} + \sup_{t\in[0,T]}\|\widetilde{Sg_2}(t,\cdot)\|^{(l)})|\mathscr{G}|^{\mathcal{X}(l)}$ under different types of conditions by the results of Proposition \ref{prop:value function}.
In particular, if $ S$ does not depend on $ a$, then
\begin{equation}\notag
\begin{aligned}
   \E\left[ \int_0^T \int_A S\partial_xV^{\top}\sigma\sigma^{\top}\partial_xV(t,\aX_t,a)\bpi(\dd a|t,\aX_t)\dd t \right] &= \E\left[ \int_0^TS(t,\aX_t)\partial_xV(t,\aX_t)^\top(\tilde\sigma^{\bpi})^2(t,\aX_t)\partial_xV(t,\aX_t)\dd t\right]\\
   &= \E\left[ \int_0^T S(t,\aX_t)\dd\left<V(\cdot,\aX_\cdot)\right>_t\right]. 
\end{aligned}
\end{equation}
This proves the desired results. 
\end{proof}

\subsection{Proofs of Examples \ref{example:general case}
and \ref{example:guassian}}
\label{sec proof: examples}

\begin{proof}[Proof of Example \ref{example:general case}]
It suffices to work on $ \tilde \sigma^{\bpi}$. Note that
\begin{equation}
    \partial_x^s (\tilde \sigma^{\bpi})^2(t,x) = \sum_{s_1 + s_2 = s}\int_\A \partial_x^{s_1} \sigma\sigma^{\top}(t,x,a)\partial_x^{s_2}H(t,x,a)  \mu(\dd a)
\end{equation}
is bounded. The desired result follows from \eqref{mideq:derivative bound example gaussain} and $ \sigma\sigma^{\top}$ is uniformly elliptic. 
\end{proof}

\begin{proof}[Proof of Example \ref{example:guassian}]
    For simplicity, we only prove that $ \tilde \sigma^{\bpi} $ has bounded derivatives with respect to $ x$ for the case $ d = m =1$. Note that 
    \begin{equation}
     \int_{\R} \sigma^2(t,x,a) \varphi(a|\mu(t,x), \Sigma(t,x)) \dd a = \int_{\R} \sigma^2(t,x,\mu(t,x) + \Sigma(t,x)u)\varphi(u)\dd u := \lambda(t,x),
    \end{equation}
    where $ \varphi$ is the standard Gaussian density. Note that 
    \begin{equation}\label{mideq:derivatives example gaussain}
    \begin{aligned}
        \partial_x \lambda^{\frac{1}{2}} &= \frac{1}{2}\lambda^{-\frac{1}{2}}\partial_x \lambda\\
        \partial_x^2 \lambda^{\frac{1}{2}} & = -\frac{1}{4}\lambda^{-\frac{3}{2}}(\partial_x \lambda)^2 + \frac{1}{2}\lambda^{-\frac{1}{2}}\partial_x^2\lambda\\
        \partial_x^3\lambda^{\frac{1}{2}} & = \frac{3}{8}\lambda^{-\frac{5}{2}}(\partial_x \lambda)^3 - \frac{3}{4}\lambda^{-\frac{3}{2}}\partial_x\lambda\partial_x^2\lambda + \frac{1}{2}\lambda^{-\frac{1}{2}}\partial_x^3\lambda\\
        \partial_x^4\lambda^{\frac{1}{2}} & = -\frac{15}{16}\lambda^{-\frac{7}{2}}(\partial_x\lambda)^4  + \frac{9}{4}\lambda^{-\frac{5}{2}}(\partial_x\lambda)^2\partial_x^2\lambda - \frac{3}{4}\lambda^{-\frac{3}{2}}(\partial_x^2\lambda)^2  -\lambda^{-\frac{3}{2}}\partial_x\lambda \partial_x^3\lambda + \frac{1}{2}\lambda^{-\frac{1}{2}}\partial_x^4\lambda
    \end{aligned}
    \end{equation}
    Under the given conditions, we can exchange the order of integration and differentiation, and obtain
    \begin{equation}
        \partial_x\lambda = \partial_x \int_{\R} \sigma^2 \varphi(u)\dd u = \int_{\R}2\sigma(\partial_x\sigma + \partial_a\sigma (\partial_x\mu + \partial_x \Sigma u)) \varphi(u) \dd u.
    \end{equation}
    It follows from H\"older's inequality that
    \begin{equation}
        |\partial_x\lambda| \le 2\sqrt{\int_{\R}\sigma^2\varphi(u)\dd u}\sqrt{\int_{\R}(\partial_x\sigma + \partial_a\sigma (\partial_x\mu + \partial_x \Sigma u))^2\varphi(u)\dd u}\le C\lambda^{\frac{1}{2}}.
    \end{equation}
    One can use the same method to obtain that for $ s\le 4$, there is a constant $ C$ such that
    \begin{equation}\label{mideq:derivative bound example gaussain}
        |\partial_x^s\lambda| \le C(1+\lambda^{\frac{1}{2}}).
    \end{equation}
    Note that $ \sigma\sigma^{\top}$ is uniformly elliptic. There is a $ \varepsilon > 0$ such that $ \lambda^{-1} \le \varepsilon^{-1}$. The desired result follows from \eqref{mideq:derivative bound example gaussain} and \eqref{mideq:derivatives example gaussain}.
\end{proof}

\section*{Acknowledgments}
Yanwei Jia is supported by the Start-up Fund at The Chinese University of Hong Kong, and Hong Kong Research Grants Council (RGC) - Early Career Scheme (ECS) 24211124. We are especially indebted to the two anonymous referees for their constructive and detailed comments that have led to an improved version of the paper. We thank conference participants at IASM Workshop on ``Bridging Theory and Practice in Finance: Stochastic Analysis and Machine Learning", Workshop on Stochastic Control and Reinforcement Learning at Harbin Institute of Technology, 2nd ETH-Hong Kong-Imperial Mathematical Finance Workshop, and 15th POMS-HK International Conference for their helpful discussions. 


\appendix

\section{Proofs of Lemmas \ref{lemma:moments of X}, \ref{lemma:convergence rate of diff} and  \ref{lemma:moments of X under Pw}}
\label{sec proof: lemma}

\begin{proof}[Proof of Lemma \ref{lemma:moments of X}]
Throughout this proof, we denote by $C\ge 0$ a generic constant independent of $t_i$, which may take a different value at each occurrence.
Note that $X^{\mathscr G}_0$ is deterministic and hence is   $p$-th power integrable for all $p\ge 2$.
Suppose that for some $i=0,\ldots, n-2$, 
\eqref{eq:sample_sde_abbre} admits a unique strong solution $X^{\mathscr G}$ on $[0,t_i]$
satisfying $\sE[\sup_{t\in [0,t_i]}|X^{\mathscr G}_t|^p]<\infty$
  for all $p\ge 2$.
Observe that  by (H.\ref{assum:standing})
and the Cauchy-Schwarz inequality,
for all $p\ge 2$,
\begin{align*}
\sE\left[\left(\int_{t_i}^{t_{i+1}} |b(t,0,a_{t_i})|\d t\right)^p\right]
&\le C (1+ \sE\left[  d_\A(a_{t_i},a_0)^p  \right])=C (1+ \sE\left[\sE\left[  d_\A(a_{t_i},a_0)^p \mid X^{\mathscr G}_{t_i} \right]\right])
\\
&=C \left(1+ \sE\left[  \int_\A  d_\A(a,a_0)^p \bpi(da|t_i,  X^{\mathscr G}_{t_i} )\right]\right)
\le  C \left(1+  \sE\left[ |X^{\mathscr G}_{t_i}|^p \right]\right)<\infty.
\end{align*}
 Similar arguments show that 
 \begin{align*}
  \sE\left[\left(\int_{t_i}^{t_{i+1}} |\sigma(t,0,a_{t_i})|^2\d t\right)^{p/2}\right] \le   
  \sE\left[ \int_{t_i}^{t_{i+1}} |\sigma(t,0,a_{t_i})|^p\d t \right]<\infty. 
 \end{align*}
Since   $x\mapsto b(t, x,a )$ and $x\mapsto  \sigma(t, x,a )$ are Lipschitz continuous uniformly in $(t,a)$, 
by   \cite[Theorems 3.3.1 and  3.4.3]{zhang2017backward},
\eqref{eq:sample_sde_abbre} with the initial state $X^{\mathscr G}_{t_i}$ at $t_i$ admits a unique strong solution $X^{\mathscr G}$ on $[t_i, t_{i+1}]$
satisfying $\sE[\sup_{t\in [t_{i},t_{i+1}]}|X^{\mathscr G}_t|^p]<\infty$
  for all $p\ge 2$. This along with an inductive argument shows  the desired well-posedness result. 
 
    Apply the notation $\delta(u):=t_i$ for $u\in[t_i,t_{i+1})$, then it satisfies that
    \begin{equation}
        \X_s = x_0 + \int_t^s b(u,\X_u,\act_{\delta(u)}) \dd u + \int_t^s \sigma(u,\X_u,\act_{\delta(u)}) \dd W_u.
    \end{equation}
    Given that $b(u,\X_u,\act_{\delta(u)})$ and $\sigma(u,\X_u,\act_{\delta(u)})$ have finite moments, there is a constant $C_p > 0$, depending only on $p$, such that
    \begin{equation}\label{mideq:333}
    \begin{aligned}
        \E\left[ \left|\X_s\right|^{p}\right] &\le C_p\left|x_0\right|^{p}+ C_p\int_t^s \E\left[ \left| b(u,\X_{u},\act_{\delta(u)})\right|^{p}\right] \dd u\\
        &+ C_p\int_t^s\E\left[\left| \sigma(u,\X_{u},\act_{\delta(u)})\right|^{p}\right]\dd u.
        \end{aligned}
    \end{equation}
    Noting that by (H.\ref{assum:standing}), for $h = b$ or $\sigma$, it satisfies that
    \begin{equation}\label{mideq:444}
    \begin{aligned}
        \E\left[ \left| h(u,\X_{u},\act_{\delta(u)})\right|^{p}\right] &\le \E\left[ C_p\left( 1+\left|\X_{u}\right|^{p} + d_\A(a_0,\act_{\delta(u)})^{p}\right)\right]\\
        &\le C_p\left( 1+\E\left[ \left|\X_{u}\right|^{p}\right]+\E\left[ \left|\X_{\delta(u)}\right|^{p}\right]\right),
        \end{aligned}
    \end{equation}
    where the second inequality follows from taking conditional expectation $\E[\cdot|\X_{\delta(u)}]$ and using Condition (H.\ref{assum:standing}\ref{item:pi}). Substituting~\eqref{mideq:444} into~\eqref{mideq:333}, we obtain
    \begin{equation}\label{mideq:555}
        \begin{aligned}
            \E\left[ \left|\X_s\right|^{p}\right] \le C_p\left|x_0\right|^{p} + C_p\int_t^s\left( 1+ \E\left[ \left|\X_{u}\right|^{p}\right]+\E\left[ \left|\X_{\delta(u)}\right|^{p}\right]\right) \dd u.
        \end{aligned}
    \end{equation}
    Choosing $s = \delta(s)$ in~\eqref{mideq:555}, and adding it with~\eqref{mideq:555}, we obtain
    \begin{equation}\notag
        \begin{aligned}
            \E\left[ \left|\X_s\right|^{p}\right] &+\E\left[ \left|\X_{\delta(s)}\right|^{p}\right]\\
            \le& 2C_p\left|x_0\right|^{p} + 2C_p\int_t^s\left( 1+ \E\left[ \left|\X_{u}\right|^{p}\right]+\E\left[ \left|\X_{\delta(u)}\right|^{p}\right]\right) \dd u.
        \end{aligned}
    \end{equation}
    The desired result follows from the Gr\"onwall's inequality.
\end{proof}

\begin{proof}[Proof of Lemma \ref{lemma:convergence rate of diff}]
        Assume that $s\in [t_{i-1},t_{i}]$, then $\delta(s) = t_{i-1}$. By  It\^o's formula, \begin{equation}
        \begin{aligned}
            &\E[ g(s,\X_s,a_{t_{i-1}}) - g(t_{i-1},\X_{t_{i-1}},a_{t_{i_1}})]\\ 
            &= \E\bigg[\int_{t_{i-1}}^s \left(\partial_t g(u,\X_u,a_{t_{i-1}}) + b^{\top}(u,\X_u,a_{t_{i-1}})\partial_x g(u,\X_u,a_{t_{i-1}})\right) \dd u\\
            &\quad + \int_{t_{i-1}}^s \frac{1}{2} \tr\left( \sigma\sigma^{\top}(u,\X_u,a_{t_{i-1}})\partial_x^2 g(u,\X_u,a_{t_{i-1}})\right) \dd u \bigg].
        \end{aligned}
        \end{equation}
        It follows from the linear growth condition of $b$ and $\sigma$ that
        \begin{equation}
        \begin{aligned}
            &\left| \E[ g(s,\X_s,a_{t_{i-1}}) - g(t_{i-1},\X_{t_{i-1}},a_{t_{i_1}})]\right|\\
            &\le \int_{t_{i-1}}^s C\|g\|_{C_p^{2,0}}\E\big[(1+|\X_u|^p + d_A(a_{t_{i-1}},a_0)^p)\big( 1 +(1+|\X_u|+d_A(a_{t_{i-1}},a))\\
            &\qquad\qquad\qquad\qquad\qquad\qquad\qquad\qquad\qquad\qquad+ \frac{1}{2}(1+|\X_u| + d_A(a_{t_{i-1}},a))^2 \big) \big]\dd u\\
            &\le C_p\|g\|_{C_p^{2,0}}|\mathscr{G}|,
            \end{aligned}
        \end{equation}
        for a constant  $C_p$  depending only on $b,\sigma$ and $p$, where  the last inequality used Lemma \ref{lemma:moments of X}.
\end{proof}

\begin{proof}[Proof of Lemma \ref{lemma:moments of X under Pw}]
        We first prove $\Ew[\sup_{t\in [0,T]} |\X_t|^p] < \infty $ for all $p \ge 2$. By (H.\ref{assump:linear growth on x}), for all $p\ge 2$,
        \begin{equation}
            \begin{aligned}
                \Ew\left[ \left( \int_{t_i}^{t_{i+1}}|b(t,0,\act_{t_i})|\dd t\right)^p\right] \le C < \infty.
            \end{aligned}
        \end{equation}
        Similarly, we have
        \begin{equation}
            \Ew\left[ \left( \int_{t_i}^{t_{i+1}}|\sigma(t,0,\act_{t_i})|^2\dd t\right)^{p/2}\right] \le \Ew\left[ \int_{t_i}^{t_{i+1}}|\sigma(t,0,\act_{t_i})|^p \dd t\right] \le C < \infty.
        \end{equation}
        Note that $\Ew\left[ \cdot\right] = \E[\cdot| \widetilde{\mathcal{F}}_0]$, and all It\^o integrals appear here equal to 0 when taking the conditional expectation $\E[\cdot| \widetilde{\mathcal{F}}_0]$. Analogous to the proofs of Lemma~\ref{lemma:moments of X} and \cite[Theorem 3.4.3]{zhang2017backward}, we obtain that $\Ew[\sup_{t\in [0,T]} |\X_t|^p] < \infty $ for all $p \ge 2$. 
        
        Next, we   prove that the moments of $\X_t$ is independent of $|\mathscr{G}|$ and obtain~\eqref{eq:moments of X under Ew}. Use the notation $\delta(u):=t_i$ for $u\in[t_i,t_{i+1})$, then it satisfies that
    \begin{equation}
        \X_s = x_0 + \int_0^s b(u,\X_u,\act_{\delta(u)}) \dd u + \int_0^s \sigma(u,\X_u,\act_{\delta(u)}) \dd W_u.
    \end{equation}
    Note that $b(u,\X_u,\act_{\delta(u)})$ and $\sigma(u,\X_u,\act_{\delta(u)})$ have finite moments by (H.\ref{assump:linear growth on x}). There is a constant $C_p > 0$, depending only on $p$, such that
    \begin{equation}\label{mideq:moments of X under Ew using others results}
    \begin{aligned}
        \Ew\left[ \left|\X_s\right|^{p}\right] &\le C_p\left|x_0\right|^{p}+ C_p\int_0^s \Ew\left[ \left| b(u,\X_{u},\act_{\delta(u)})\right|^{p}\right] \dd u\\
        &+ C_p\int_0^s\Ew\left[\left| \sigma(u,\X_{u},\act_{\delta(u)})\right|^{p}\right]\dd u.
        \end{aligned}
    \end{equation}
    Noting that by condition (H.\ref{assump:linear growth on x}), for $h = b$ or $\sigma$, it satisfies that
    \begin{equation}\label{mideq:moments of h}
    \begin{aligned}
        \Ew\left[ \left| h(u,\X_{u},\act_{\delta(u)})\right|^{p}\right] &\le \Ew\left[ C_p\left( 1+\left|\X_{u}\right|^{p}\right)\right].
        \end{aligned}
    \end{equation}
    Substituting~\eqref{mideq:moments of h} into~\eqref{mideq:moments of X under Ew using others results}, we obtain
    \begin{equation}
        \begin{aligned}
            \Ew\left[ \left|\X_s\right|^{p}\right] \le C_p\left|x_0\right|^{p} + C_p\int_t^s\left( 1+ \Ew\left[ \left|\X_{u}\right|^{p}\right]\right) \dd u.
        \end{aligned}
    \end{equation}
    The desired result follows from the Gr\"onwall's inequality.
    \end{proof}

\end{document}